\documentclass{article}

\usepackage[preprint]{neurips_2026}

\usepackage[utf8]{inputenc} 
\usepackage[T1]{fontenc}    
\usepackage{hyperref}       
\usepackage{url}            
\usepackage{booktabs}       
\usepackage{amsfonts}       
\usepackage{nicefrac}       
\usepackage{microtype}      
\usepackage{xcolor}         
\usepackage{amsmath,amssymb,amsthm, bm}
\usepackage{graphicx}
\usepackage{booktabs}
\usepackage{caption}
\usepackage{wrapfig}
\usepackage{subcaption}
\usepackage{enumitem}
\usepackage{hyperref}
\usepackage{cleveref}
\usepackage{siunitx}
\usepackage{tikz}
\newtheorem{theorem}{Theorem}
\theoremstyle{definition}

\newtheorem{corollary}{Corollary}[theorem]

\newtheorem{proposition}{Proposition}

\usepackage{algpseudocode}
\usepackage{algorithm}
\usepackage{appendix}
\title{Spectral Geometry of Attention:\\ From Information Routing to Uncertainty}

\author{%
  Giulio Viganò\thanks{Work performed while at Ecole Polytechnique} \\
  University of Milano-Bicocca\\
  Milan, Italy \\
  \texttt{g.vigan36@campus.unimib.it} \\
  \And
  Simone Melzi \\
  University of Milano-Bicocca\\
  Milan, Italy \\
  \And
  Maks Ovsjanikov \\
  LIX, Ecole Polytechnique\\
  IP Paris \\
}
\begin{document}
\maketitle

\begin{abstract}
In this work, we study transformer attention through the lens of spectral geometry and operator theory. We view each attention head as a functional map between Hilbert spaces of functions on the token sequence and derive a \emph{Token Difference Operator}, whose spectral structure controls how token-space information is routed to the output. We show that standard Euclidean spectra are structurally biased by sinks, conflating mass concentration with genuine routing capacity. By recasting token space in the intrinsic probability geometry induced by attention, the token difference spectrum disentangles sink effects from routing capacity and provides a spectral description of the dimensionality of the head output. This yields a unified framework for analyzing attention maps, explaining sinks, routing collapse, and output dimensionality within a single operator-theoretic framework. 
In practice, by grounding attention heuristics in spectral geometry, we develop a novel attention-based uncertainty estimator that complements probability-based scores, with the largest gains on long-context inputs.
\end{abstract}
\section{Introduction}\label{sec:intro}

Attention maps are increasingly viewed as windows into the internal computation of transformer architectures \cite{vaswani2023attentionneed}. Prior works have shown that attention carries useful signals about how transformers process information, studying it as probability distributions~\cite{vaswani2023attentionneed}, graph operators~\cite{noël2025graphsignalprocessingframework,binkowski-etal-2025-hallucination}, transition matrices~\cite{erel2025attentionasdiscretetimemarkov}, or through row-wise statistics~\cite{vazhentsev2025uncertaintyawareattentionheadsefficient,zhai2023stabilizingtransformertrainingpreventing}. While these perspectives have revealed important insights, they primarily focus on pairwise token interactions, or leveraging models that do not describe the single pass effect of the attention, leaving comparatively unexplored the global structural properties of attention operators as a whole, especially in relation to the hidden representation dimensionality and the final output of the transformer. 

In this work, we approach attention from this global operator perspective, framing each attention map $A$ as a linear functional operator that transfers signals between geometric domains defined for the token sequence, and investigating its spectral properties. Inspired by functional maps~\cite{ovsjanikov20212fmap} and shape difference operators~\cite{rustamov2013shapediff,huang2019shapediff}, we derive a functional operator that characterizes the single-step action $V \mapsto AV$ induced by an attention head. Its spectrum reveals the multi-scale, global behavior of attention and is, as we demonstrate, strongly correlated with a range of phenomena from attention sinks to model uncertainty.

We introduce the \emph{Token Difference Operator} (TDO), a framework that departs from standard Euclidean analysis by assigning distinct geometric structures to the input and output domains of the attention map $A$. We demonstrate that the conventional practice of adopting a uniform Euclidean metric is structurally flawed; specifically, it is biased by the inherent row-stochasticity of attention and the presence of attention sinks. In these biased settings, the mass assigned to initial tokens can saturate the operator spectrum, making it appear that there is significant information flow when no actual token-dependent routing is occurring.
To rectify this, we induce an intrinsic probability geometry on the input domain, directly derived from the attention mechanism.
With this intrinsic metric, the trivial constant mode corresponds to a transported mean, while the remaining centered modes describe the token-dependent component of the head output. This yields a mean--routing decomposition of $AV$ that cleanly separates phenomena often conflated in the literature, including attention sinks~\cite{gu2025attentionsinkemergeslanguage,barbero2025llmsattendtoken,ruscio2025sinkinggeometricapproachattention}, dimensional collapse~\cite{queipodellano2026attentionsinkscompressionvalleys,sun2026spikesparsesinkanatomy}.

Importantly, these structural quantities are not purely conceptual constructions: the TDO can be computed from attention maps, and its spectrum provides a strong signal in a range of probing tasks. We show that intrinsic spectral diagnostics are less dominated by sink score than raw attention or Euclidean spectral analysis, while aligning more closely with the dimensionality of the head output. Building on this, we introduce \emph{Spectral Attention Uncertainty Quantification} (SAUQ), an unsupervised single-pass uncertainty estimator. Unlike probability-based estimators~\cite{zhang2024luqlongtextuncertaintyquantification,kuhn2023semanticuncertaintylinguisticinvariances,duan2024shiftingattentionrelevancepredictive}, which can fail when task-relevant distinctions are compressed before reaching the logits~\cite{barbero2025glasses,pasten2025continuityisolationleaddoubts,velickovic2026perplexitytellrightwrong}, SAUQ looks upstream at routing capacity and improves over closely related attention-based methods, especially on long-form summarization tasks. 

We formalize these insights through the following theoretical and practical contributions, also summarized in Figure \ref{fig:teaser}:
\begin{itemize}[leftmargin=*,noitemsep,topsep=0pt]
    \item \textbf{A global operator-theoretic view of attention.} We model attention maps as linear functional operators on token-space signals, shifting focus from token-level summaries to the spectral structure of the TDO.
    \item \textbf{A theoretical separation between sinks and dimensional collapse.} We prove that Euclidean spectra can be biased by sink tokens, and derive an intrinsic probability geometry in which the information dynamic is decomposed into a mean mode and a token-dependent routing subspace, proving that the nontrivial spectrum controls the dimensionality of the head output.
    \item \textbf{A link to uncertainty}: We empirically show that the spectral signal of high routing heads can be leveraged to predict uncertainty of LLMs.
\end{itemize}

\begin{figure}[!t]
    \centering
    \includegraphics[width=\linewidth,
        clip]{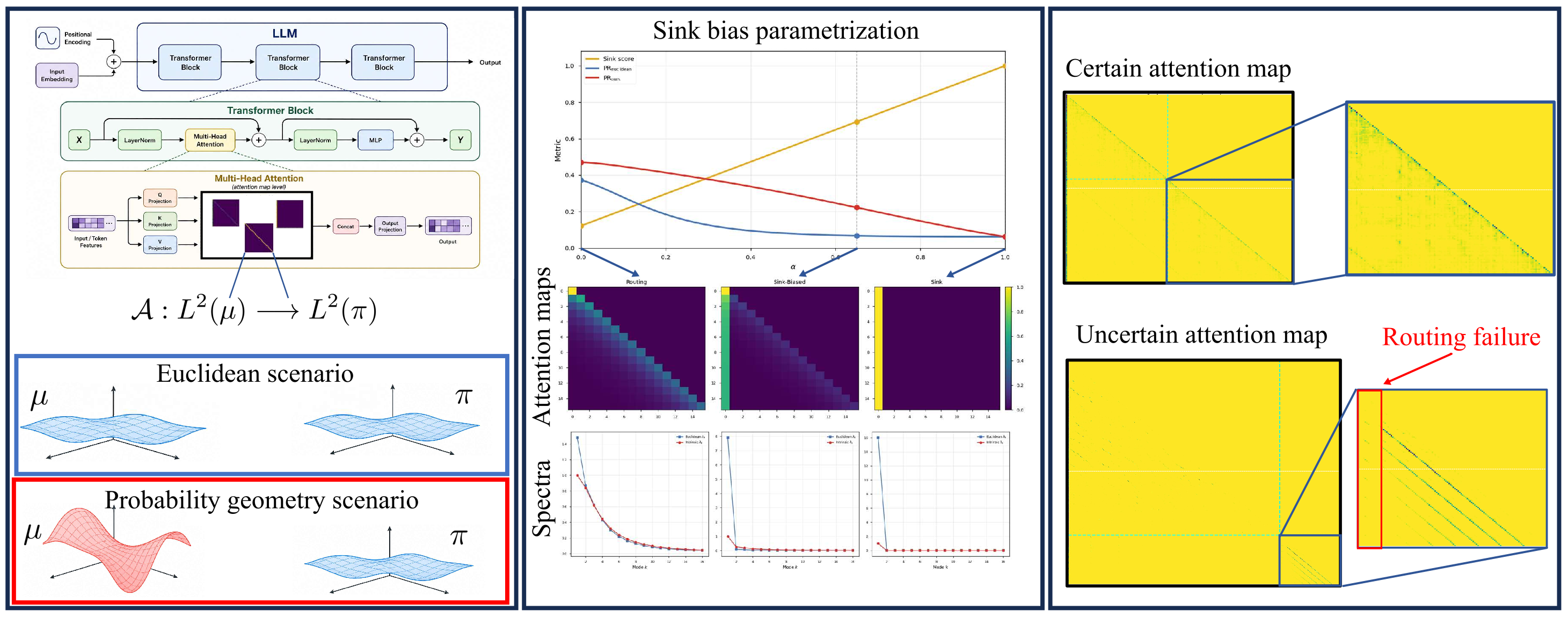}
    \caption{Teaser: (Left block) Our work studies transformers' attention as a functional map between Hilbert spaces equipped with an inner product. The properties of the operator change with the standard Euclidean measure (Blue) and with an intrinsic probability geometry measure (Red). (Center Block) Parametrizing row stochastic causal map interpolating ($\alpha$) between a copy map and a sink map, we show that the Euclidean spectrum of the matrix is biased towards the sink, while reweighting the input measure relaxes this bias. (Right Block) We leverage our insights to design an uncertainty quantifier, identifying uncertainty as a routing capability collapse in certain selected heads (red box in the uncertain model).}
    \label{fig:teaser}
\end{figure}

\section{Related Work}
\label{sec:related}

\paragraph{Attention Interpretability.}
A broad line of work studies attention as a mechanism for information processing and propagation in transformers. The classical view interprets attention weights as soft similarity scores between queries and keys~\cite{vaswani2023attentionneed}. More recent approaches model attention as a discrete-time Markov chain~\cite{erel2025attentionasdiscretetimemarkov}, as a graph signal operator with an associated Laplacian~\cite{noël2025graphsignalprocessingframework, binkowski-etal-2025-hallucination, tian2025spectralkdunifiedframeworkinterpreting}, or through spectral properties such as QK eigenspectrum concentration and localization~\cite{bao2024selfattentionnetworkslocalizeqkeigenspectrum}. These methods reveal important structure, but they typically analyze attention as a probability matrix, graph, or repeated operator. Our work instead focuses on the single-step routing map $V \mapsto AV$ and studies its spectrum with respect to an intrinsic geometry induced by attention itself.

This contribution is more relevant in the presence of attention sinks and representation collapse. Attention sinks have been studied through embedding spectra~\cite{cancedda-2024-spectral}, information propagation~\cite{barbero2025llmsattendtoken}, and geometric analyses~\cite{ruscio2025sinkinggeometricapproachattention}, and related phenomena have also been observed in vision transformers~\cite{darcet2024visiontransformersneedregisters}. In parallel, hidden-state collapse and compression effects have been linked to oversquashing and attention sinks~\cite{barbero2025glasses,queipodellano2026attentionsinkscompressionvalleys}, as well as to spiking behavior in feed-forward layers~\cite{sun2026spikesparsesinkanatomy, dong2023attentionneedpureattention}. We build on this line of work by showing that standard Euclidean spectral diagnostics can confuse sink concentration with routing capacity, and by introducing an intrinsic spectrum that separates these effects.

\paragraph{Uncertainty quantification in LLMs.}
Uncertainty estimation and hallucination detection for LLMs are often based on output probabilities, using quantities such as entropy, perplexity, likelihood, or sampling-based semantic consistency~\cite{zhang2024luqlongtextuncertaintyquantification,kuhn2023semanticuncertaintylinguisticinvariances,duan2024shiftingattentionrelevancepredictive}, and also spectral signals computed on the output logits \cite{lin2024generatingconfidenceuncertaintyquantification, vipulanandan2026semantic}. However, probability-based signals can be unreliable when task-relevant distinctions are compressed before reaching the logits~\cite{barbero2025glasses,pasten2025continuityisolationleaddoubts,velickovic2026perplexitytellrightwrong}. This motivates uncertainty signals derived from internal model dynamics.

For this reason, recent methods leverage hidden representations of LLMs to quantify uncertainty~\cite{kumaran2026llmscomputeverbalconfidence, chen2024insidellmsinternalstates}. In particular, some works focus on attention heads, looking at attention values~\cite{Sriramanan2024_attention_score}, previous-token relations~\cite{vazhentsev2025uncertaintyawareattentionheadsefficient}, Laplacian eigenvalues of attention graphs~\cite{binkowski-etal-2025-hallucination}, and topological divergence measures on attention graphs~\cite{bazarova2025hallucinationdetectionllmstopological}. 
While closely related to these attention-based methods, our approach replaces scalar or graph-based summaries with the intrinsic spectrum of the TDO.
Our solution grounds uncertainty estimation in the same geometry that controls routing capacity, yielding a single-pass estimator that captures spectral collapse in token-level information flow.
\section{Theoretical Framework}
\label{sec:theory}
We consider a decoder-only transformer with $L$ layers and $H$ attention heads per layer. The attention head $h$ at layer $l$ is computed as
\begin{equation}
    A_{l,h}
    =
    \mathrm{softmax}\!\left(Q_{l,h}K_{l,h}^\top/\sqrt{d} + C\right),
    \label{eq:attn}
\end{equation}
and acts on the value matrix as $A_{l,h}V_{l,h}$, where $C$ is the causal mask, $Q_{l,h},K_{l,h},V_{l,h}\in \mathbb{R}^{n,d}$.
For each layer $l$, the full multi-head attention output is obtained by concatenating the heads and projecting the output $
    \mathrm{MHA}_l(\hat{X}_{l-1})
    =
    \text{concat}_{h}\!\left(A_{l,h}V_{l,h}\right)W^O_l .$
The complete transformer block notation is provided in the Supplementary material.  Our study focuses on the matrix $A_{l,h}$ acting on the token space as a token-routing operator. 
We analyze one head at a time and omit the layer and head indices when no confusion arises.
For a fixed input, each attention map $A\in\mathbb R^{n\times n}$, as an operator acting on $V$, is row-stochastic, causal, and linear in the columns of its input $V$. These properties correspond to:
\begin{equation}
    A_{ij}\ge 0 \ \forall i, j;
    \qquad
    \sum_{j=1}^n A_{ij}=1 \ \forall i;
    \qquad
    A_{ij}=0 \ \forall j>i.
    \label{eq:attention-properties}
\end{equation}

\subsection{Attention as a functional mapping and Token Difference Operator}
\label{subsec:attention-functional-mapping}

Let $\mathcal T=\{1,\ldots,n\}$ denote the space of tokens described by their position. We define the weighted Hilbert space $L^2(\mu)=L^2(\mathcal T,\mu)$, as the space $\mathcal{T}$ equipped with the inner product:
\begin{equation}
    \langle f,g\rangle_\mu
    =
    \sum_{j=1}^n \mu_j f_j g_j
    =
    f^\top \Omega_\mu g,
    \qquad
    f,g\in L^2(\mathcal T)=\mathbb R^n,
    \label{eq:weighted-inner-product}
\end{equation}
where $\mu$ is a positive measure on $\mathcal T$, which can be encoded as a vector of length $n$ and $\Omega_\mu=\operatorname{diag}(\mu)$. Given another positive measure $\pi$ on $\mathcal T$, we can interpret the attention matrix $A$ as a linear map
\begin{equation}
    \mathcal{A}:L^2(\mu)\longrightarrow L^2(\pi),
    \qquad
    \mathcal{A}(f) = Af.
    \label{eq:T-def}
\end{equation}

Assuming $\mu_j>0$ for all $j$, the operator $\mathcal A$ admits an adjoint $\mathcal A^*:L^2(\pi)\to L^2(\mu)$ defined by $ \langle \mathcal A f,g\rangle_\pi
    =
    \langle f,\mathcal A^*g\rangle_\mu 
$, which can be computed with the following formula:
\begin{equation}
    \mathcal A^*
    =
    \Omega_\mu^{-1}A^\top \Omega_\pi .
    \label{eq:adjoint}
\end{equation}

Having established these two operators and inspired by spectral geometry \cite{huang2019shapediff}, we define the \emph{Token Difference Operator} (TDO) as 
\begin{equation}
    D_{\mu,\pi}: L^2(\mu)\to L^2(\mu) \qquad
    D_{\mu,\pi}
    := \mathcal{A}^* \mathcal{A}
    = \Omega_\mu^{-1} A^\top \Omega_\pi A,
    \label{eq:weighted-token-difference}
\end{equation}
This operator is self-adjoint and positive semidefinite in $L^2(\mu)$, and is strictly related to the Rayleigh quotient $\mathcal{R}_{\mu,\pi}(f) = \|\mathcal{A} f\|_\pi^2 / \|f\|_\mu^2$, which measures how much the attention operator $A$ stretches an input signal $f$ by acting on it and transferring it between the two geometries. The properties of the TDO strongly depend on the measures $\mu,\pi$, which we refer to as the geometry, and we will discuss this relation in the next sections. 

We note that all spaces considered here are finite-dimensional: $L^2(\mu)$ is $\mathbb R^n$ equipped with a weighted inner product, and every result below can be written in matrix form. We adopt the operator language because it makes the construction discretization-agnostic, as for functional maps~\cite{ovsjanikov20212fmap} and shape difference operators~\cite{rustamov2013shapediff}; concurrent work~\cite{xiao2026functionalattentionpairwiseaffinities} similarly treats attention as a map between function spaces.


\subsubsection{Euclidean scenario: structural sink bias}
\label{subsec:counting}

Taking both $\mu$ and $\pi$ as counting measures (i.e. $\mu_j = \pi_j = 1 \ \forall j$) gives $D_{\mu,\pi} = A^\top A$, the standard Euclidean Gram matrix from which the standard singular values $\{\sigma_i(A)\}_{i=1}^{n}$ are computed. In this geometry, row-stochastic attention is generally not a contraction.

\begin{proposition}[Euclidean Non-contractivity]
\label{prop:counting-noncontractive}
For $A$ row-stochastic, $\sigma_1(A) \ge 1$, with equality if and only if $A$ is doubly stochastic.
\end{proposition}

The constraint $A\mathbf{1}=\mathbf{1}$ fixes the constant function as a
unit-gain direction for $A$, but in counting geometry the TDO acts on this mode
as $
    A^\top A\mathbf{1}=A^\top\mathbf{1}=c, c_j=\sum_i A_{ij}.$
Thus, the Euclidean Gram operator converts the constant mode into the
column-sum profile. Under causal masking $A_{11}=1$, hence
$c_1\ge 1$ and $\sum_j c_j=n$; early keys are also visible to more query
positions than late keys. The leading Euclidean spectrum is therefore sensitive
to accumulated sink mass. In the extreme causal sink
$A=\mathbf{1}e_1^\top$, this becomes
$A^\top A=n e_1e_1^\top$, so $\sigma_1(A)^2=n$ and all spectral energy lies in
the sink coordinate. The bias is not only a rank error, but a quadratic amplification of column-mass concentration.
The proof of Proposition~\ref{prop:counting-noncontractive} and a Frobenius budget refining this picture appear in Supplementary materials~\ref{app:counting}.

\subsubsection{Probability geometry scenario: intrinsic non-expansion}
\label{subsec:mixed}
The non-contractivity above reflects a mismatch between row-stochasticity and an unweighted geometry: the counting measure ignores the nonuniform column mass induced by causal attention. We resolve this by taking the output measure to be probabilistic and pulling it back to obtain the input measure:
\begin{equation}
    \pi = \frac{1}{n}\mathbf{1},
    \qquad
    \mu = A^\top \pi = \frac{1}{n}\sum_{i=1}^n (A_{i:})^{\top}.
    \label{eq:mixed-measures}
\end{equation}
Both are probability measures; $\mu$ is the average attention row over keys. This choice also restores the correct stationary geometry of the TDO. Since $\mu=A^\top\pi$, we have $D_{\mu,\pi}\mathbf 1=\mathbf 1,\mu^\top D_{\mu,\pi}=\mu^\top .$ Thus, the pulled-back column-mass profile is not amplified as an unweighted Euclidean direction; it becomes the invariant measure with respect to which
constants and centered fluctuations are defined.

\begin{proposition}[Intrinsic non-expansion]
\label{thm:sigma1}
With the measures \eqref{eq:mixed-measures}, $\mathcal{A}: L^2(\mu) \to L^2(\pi)$ satisfies $\|\mathcal{A} f\|_\pi \le \|f\|_\mu$ for all $f \in L^2(\mu)$, with operator norm $1$ and $\lambda_1(D_{\mu,\pi}) = 1$ attained by constant $\mathbf{1}$.
\end{proposition}

\emph{Sketch of proof.} The bound follows from Jensen's inequality $(Af)_i^2 \leq \sum_j A_{ij} f_j^2$ averaged with uniform $\pi$; equality on $\mathbf{1}$ is trivial. Full proof in Supplementary materials~\ref{app:thm-sigma1}.

\begin{corollary}[Centered contraction]
\label{cor:centered-contraction}
Let $\mathbf{1}^{\perp_\mu} = \{f \in L^2(\mu) : \mu^\top f = 0\}$ and order
$1 = \lambda_1 \ge \lambda_2 \ge \cdots$. Then
\begin{equation}
    \sup_{f \in \mathbf{1}^{\perp_\mu} \setminus \{0\}}
    \frac{\|\mathcal{A} f\|_\pi}{\|f\|_\mu}
    = \sqrt{\lambda_2(D_{\mu,\pi})}
    \le 1.
    \label{eq:centered-contraction}
\end{equation}
\end{corollary}

The trivial unit-gain mode is the constant function. All non-trivial routing happens on its $\mu$-orthogonal complement, where the contraction factor is governed by $\lambda_2(D_{\mu,\pi})$. We report the complete proof in the Supplementary materials~\ref{app:cor-centered}.

The unit singular mode is still present, but it now has a fixed interpretation: it is a constant function. This mode carries no token-dependent information, since it assigns the same value to every position and is preserved by any row-stochastic attention map. Therefore, the relevant question is not whether
$A$ preserves constants, but how much variation around the attention-weighted mean can survive one attention step.

\subsection{From functions to representations}
\label{subsec:routed-representations}

The preceding analysis concerns scalar functions $f\in L^2(\mathcal T,\mu)$. The value matrix $V\in\mathbb R^{n\times d}$ can be
viewed columnwise as $d$ such functions, hence, the intrinsic contraction result
directly describes how an attention head routes hidden representations across
token positions. We decompose $V$ into its $\mu$-mean and its centered component:
\begin{equation}
    m := \mu^\top V \in \mathbb{R}^d,
    \qquad
    V_c := V - \mathbf{1} m,
    \qquad
    \mu^\top V_c = 0.
    \label{eq:value-decomposition}
\end{equation}
Here, $m$ is the attention-weighted mean value vector, while $V_c$ contains the
token-dependent deviations around this mean.

\begin{proposition}[Mean--routing decomposition]
\label{thm:mean-routing-decomposition}
Under the intrinsic measures \eqref{eq:mixed-measures}, the head output
decomposes as
\begin{equation}
    AV = \mathbf{1} m + A V_c,
    \label{eq:mean-routing}
\end{equation}
with $\langle \mathbf{1} m, A V_c \rangle_{\pi, F} = 0$ and the routed
component bounded by
\begin{equation}
    \|A V_c\|_{\pi, F}
    \le \sqrt{\lambda_2(D_{\mu,\pi})} \, \|V_c\|_{\mu, F}.
    \label{eq:routed-contraction}
\end{equation}
where we define $\|U\|_{\mu,F}^2 = \operatorname{tr}(U^\top \Omega_\mu U)$ as the weighted Frobenius norm.
\end{proposition}

\emph{Sketch of proof.} $V = \mathbf{1} m + V_c$ and $A\mathbf{1} = \mathbf{1}$
gives the split; orthogonality is trivial, and the norm bound applies
Corollary~\ref{cor:centered-contraction} columnwise. Full proof in
Supplementary materials~\ref{app:thm-mean-routing}.

The choice of intrinsic measures aligns the barycenters of the input and output
representations. Since $\pi^\top A=\mu^\top$, we have
\begin{equation}
    \pi^\top AV = \mu^\top V = m.
\end{equation}
Thus, attention maps the $\mu$-weighted input mean to the $\pi$-weighted output
mean. Thus, \(\mathbf{1}m\) captures the position-independent transported
mean, while \(AV_c\) captures token-dependent routing. The factor
\(\sqrt{\lambda_2(D_{\mu,\pi})}\) controls how much centered variation
survives one attention step, and therefore governs the routed
dimensionality, as shown in the following Corollary.
\begin{corollary}[Spectral entropy bound]
\label{cor:spectral-entropy-routed}
Let $Z = \Omega_\pi^{1/2} AV$, $Z_c = \Omega_\pi^{1/2} A V_c$, and
$\alpha = \|Z_c\|_F^2 / (\|m\|_2^2 + \|Z_c\|_F^2)$. Define the spectral
entropy $H(Y) = -\sum_k p_k \log p_k$ with
$p_k = \sigma_k(Y)^2 / \|Y\|_F^2$. Then, with $h_2$ the binary entropy,
\begin{equation}
    \alpha H(Z_c) \le H(Z) \le h_2(\alpha) + \alpha H(Z_c),
    \qquad
    \|Z_c\|_F \le \sqrt{\lambda_2(D_{\mu,\pi})} \|V_c\|_{\mu, F}.
    \label{eq:spectral-entropy-bound}
\end{equation}
\end{corollary}

The spectrum of $D_{\mu,\pi}$ thus controls the dimensionality of the routed
component in two complementary senses: as a norm constraint and as the spectral entropy of $Z_c$ entering the
mixture bound on $H(Z)$. Proof in Supplementary materials~\ref{app:cor-spectral-entropy}.

\subsection{Spectral metrics}
\label{subsec:participation-ratio}
The TDO spectrum describes how centered token-dependent information survives an
attention step. We summarize it using the normalized participation ratio
\begin{equation}
    \mathrm{PR}(D_{\mu,\pi})
    := \frac{1}{n}
    \frac{\bigl(\sum_k \lambda_k\bigr)^2}{\sum_k \lambda_k^2}
    =
    \frac{1}{n}
    \frac{\operatorname{tr}(D_{\mu,\pi})^2}{\operatorname{tr}(D_{\mu,\pi}^2)}.
    \label{eq:pr-definition}
\end{equation}
It ranges from $1/n$ for a single active mode to
$\operatorname{rank}(D_{\mu,\pi})/n$, for a uniform nonzero spectrum, and is invariant to
global rescaling. In practice, we compute it from the symmetric matrix.$S_{\mu,\pi}=\Omega_\mu^{-1/2}A^\top \Omega_\pi A \Omega_\mu^{-1/2},$ which is symmetric and similar to $D_{\mu,\pi}$. Large $\mathrm{PR}(D_{\mu,\pi})$ indicates that centered variation is routed through several active directions; small values indicate compression of $V_c$ before the residual stream (Corollary~\ref{cor:spectral-entropy-routed}). This statistic is used in the uncertainty estimator of Section~\ref{sec:uq}. We also consider the centered normalized participation ratio $\mathrm{PR}_{\mathrm{cent}}(D_{\mu,\pi})
    =
    \frac{1}{n-1}
    \frac{\left(\sum_{k=2}^n \lambda_k\right)^2}
         {\sum_{k=2}^n \lambda_k^2}.
$
\section{Sinks, Routing, and Output Dimensionality}
\label{sec:empirical}

We test two predictions of the theory: that the intrinsic TDO reduces sink bias, and that its centered spectrum tracks output dimensionality. We compare intrinsic spectral diagnostics ($\mathrm{PR}(D_{\mu,\pi})$ and its centered restriction) against (i) \emph{raw attention statistics}, sink score, previous-token score~\cite{vazhentsev2025uncertaintyawareattentionheadsefficient}, attention entropy~\cite{zhai2023stabilizingtransformertrainingpreventing}; and (ii) \emph{counting-geometry spectral diagnostics},log-determinant score~\cite{Sriramanan2024_attention_score}, Laplacian diagonal entropy~\cite{binkowski-etal-2025-hallucination}, and $\mathrm{PR}(A^\top A)$. All metrics are evaluated against both sink score and the output spectral entropy $H(AV)$~\cite{barbero2025glasses} across the full head population.

\paragraph{Representative heads}
\label{subsec:toy}

\begin{figure}[!t]
    \centering
    \includegraphics[width=\linewidth]{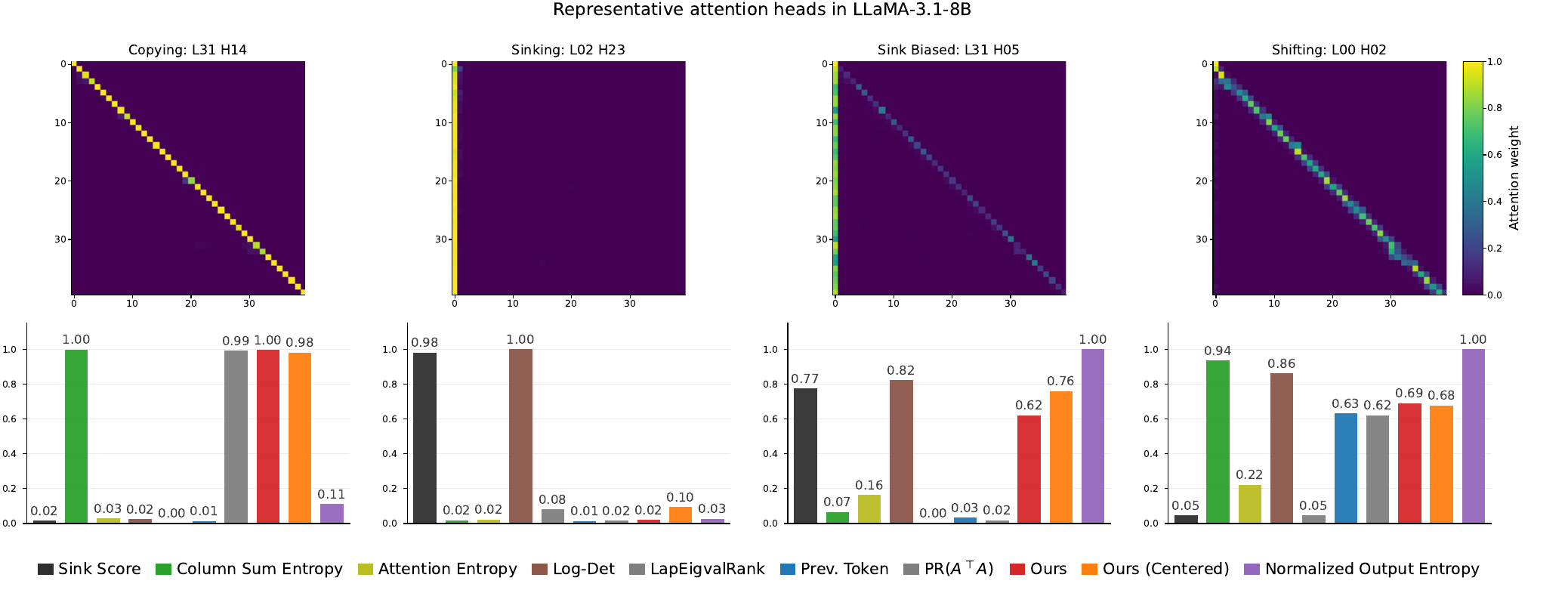}
    \caption{
    Representative attention heads and how different metrics classify their routing behavior. The intrinsic TDO score separates sink collapse from sink-biased routing, assigning high routing scores to copying, shifting, and sink-biased heads while distinguishing them from pure sink heads.}
    \label{fig:representative}
\end{figure}
Figure~\ref{fig:representative} compares four heads with clear routing properties that metrics should distinguish: a copying head, a sink head, a sink-biased head, and a shifting head. In the sink-biased head, substantial sink mass coexists with nontrivial routing, so counting-geometry and graph-based diagnostics confuse column-mass concentration with routing collapse, scoring it comparably to the sink head. The shifting head exposes the complementary failure: previous-token score~\cite{vazhentsev2025uncertaintyawareattentionheadsefficient} is tailored to this specific pattern and correctly flags it, but for the same reason misses the copying head, which routes information through a different position. In contrast, $\mathrm{PR}(D_{\mu,\pi})$ assigns high routing scores to both copying, shifting heads, and sink bias, while separating them from sink collapse, matching the routed output dimensionality across all four cases. We note that the copying head itself has low output dimensionality, highlighting that the output also depends on properties of $V$.

\paragraph{Sink dependence and output dimensionality}
\label{subsec:sink}
\begin{figure}[!t]
    \centering

    \includegraphics[width=\linewidth]{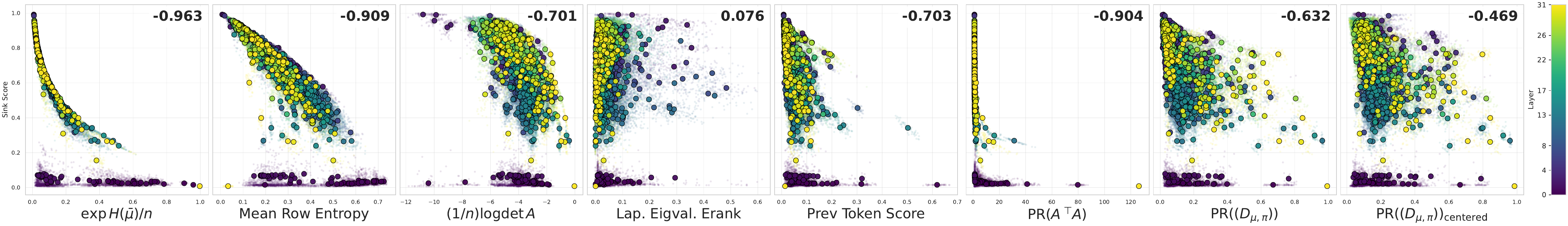}
    \includegraphics[width=\linewidth]{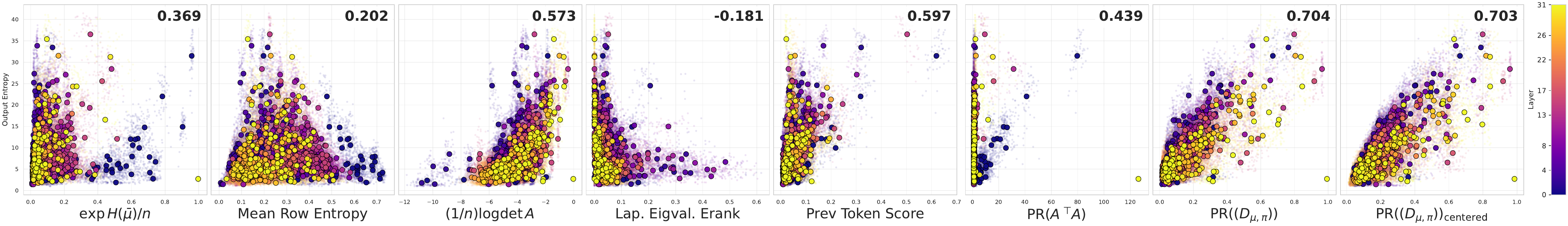}
    \caption{ Scatterplots showing the relationships between attention diagnostics, sink score, and output
    dimensionality. Top: relation between each diagnostic and sink score.
    Bottom: relation between each diagnostic and the spectral entropy of the
    head output $H(AV)$. Points correspond to the attention heads of
    LLaMA-3.1-8B on 200 inputs, shown both per input and averaged across inputs (bold points). We report the Spearman correlation between the metrics.}
    \label{fig:correlation-panels}
\end{figure}
We first refine the relation between sink concentration and routing collapse. Under causal row-stochasticity, rank deficiency of $A$ forces both phenomena jointly: a low-rank $A$ must concentrate column mass on early tokens, in the rank-one limit, $A = \mathbf{1} e_1^\top$, and the TDO spectrum loses centered modes accordingly. The converse fails: a head with imbalanced column sums can still maintain a rich centered spectrum, so sink behavior does not imply representational collapse. Disentangling the two regimes, therefore, requires a diagnostic that does not amplify column mass as a Euclidean direction. The intrinsic TDO satisfies this: its leading mode is fixed at the constant function and routing is measured on the centered subspace $\mathbf{1}^{\perp_\mu}$. Figure~\ref{fig:correlation-panels} confirms this empirically across heads. \emph{Top:} column-sum entropy, attention entropy, previous-token score, log-determinant score, and $\mathrm{PR}(A^\top A)$ all show visible dependence on sink score, with a correlation score between $0.5$ and $1$, while diagnostics derived from $D_{\mu,\pi}$ are substantially less sink-dependent; Laplacian diagonal entropy~\cite{binkowski-etal-2025-hallucination} shows the weakest sink dependence, plausibly because its degree normalization plays a role analogous to the pullback measure $\mu$. \emph{Bottom:} $\mathrm{PR}(D_{\mu,\pi})$, especially its centered variant, aligns more closely with $H(AV)$ than raw attention statistics or counting-geometry ranks, supporting Corollary~\ref{cor:spectral-entropy-routed}. The baselines fail in complementary ways: raw entropies are dominated by sink structure, while log-determinant and Laplacian-based scores identify collapsed outputs but miss the full range of output entropies. The intrinsic TDO spectrum thus best captures the non-collapsed, token-dependent routing signal needed for the uncertainty estimator developed next.
\section{Spectral Attention Uncertainty Quantification}
\label{sec:uq}

\begin{figure}[!t]
    \centering
    \includegraphics[width=\linewidth]{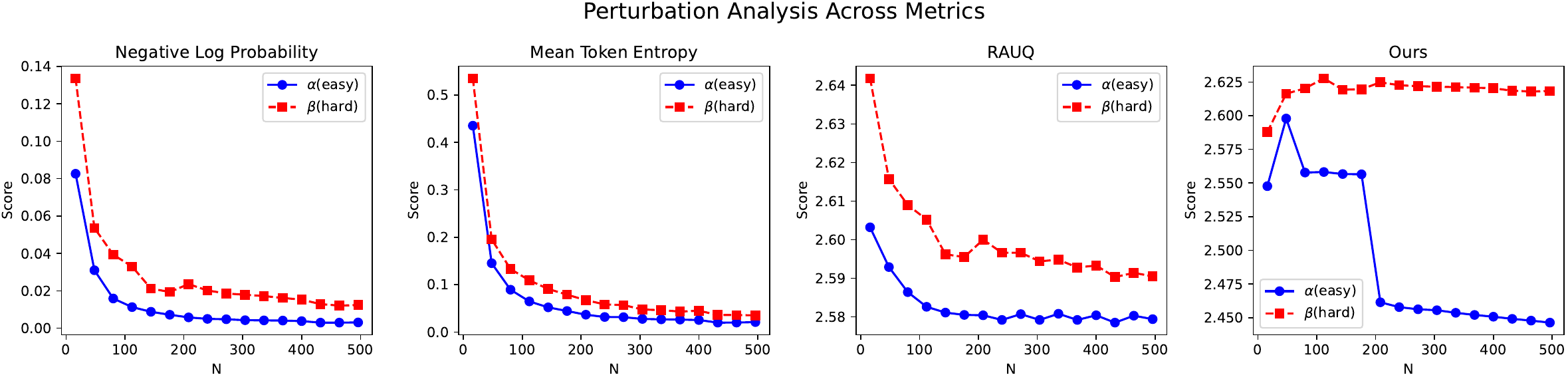}
    \caption{Perturbation sensitivity in a controlled copy task inspired
    by~\cite{velickovic2026perplexitytellrightwrong,pasten2025continuityisolationleaddoubts}.
    The model receives \texttt{copy:\,$b_1 b_2 \cdots b_n$} and must reproduce the bit
    string; we contrast the all-zero input $0^n$ ($\alpha$, easy) with a perturbation
    $b^{(r)}$ carrying a single one at position $r$ ($\beta$, hard). As $n$ grows,
    perplexity and mean token entropy become nearly indistinguishable across the two
    families, while RAUQ~\cite{vazhentsev2025uncertaintyawareattentionheadsefficient}
    and Ours, computed from hidden-layer routing geometry, remain sensitive to the
    perturbation.}
    \label{fig:perturbation}
\end{figure}
Probability-based uncertainty estimators, such as perplexity, sequence likelihood, and token entropy, are
computed at the unembedding output, after the residual stream has been mapped to logits.
When task-relevant distinctions are compressed during representation propagation, two inputs that differ in important ways induce nearly indistinguishable logits, and any output-probability estimator inherits this collapse~\cite{barbero2025glasses,pasten2025continuityisolationleaddoubts,velickovic2026perplexitytellrightwrong}, especially in long-range or non-local tasks. Figure~\ref{fig:perturbation} makes this concrete: on a controlled copy task, perplexity and token entropy become indistinguishable across two input families that still differ in a task-relevant position, while attention-based signals derived from hidden-layer routing geometry remain sensitive. 
The framework of Section~\ref{sec:theory} defines what the signal measures: the TDO spectrum quantifies the effective dimensionality of the centered, token-dependent component routed by an attention head. We hypothesize that spectral collapse during generation signals a loss of routing capacity that is not yet visible in output probabilities, and we test this hypothesis empirically. To this end, we define \emph{Spectral Attention Uncertainty Quantification} (SAUQ), a single-pass estimator that turns intrinsic attention spectra into a confidence score for selective prediction.

\subsection{Method}
\label{subsec:sauq-method}
Let $x=(x_1,\ldots,x_m)$ be the prompt and $y=(y_1,\ldots,y_T)$ the generated continuation. At the
decoding step $t$, the model attends over the prefix of length $n_t=m+t$, with attention map
$A_t^{(l,h)}\in\mathbb R^{n_t\times n_t}$ for layer $l$ and head $h$. We denote the generation-only
submatrix of $A_t^{(l,h)}$ as $B_t^{(l,h)} = A_t^{(l,h)}[G_t,G_t]$, with $G_t=[m+1,\ldots,m+t]$ the set
of generated tokens.
 
\paragraph{Head selection.}
We focus on the middle third of transformer layers, $\mathcal L_{\mathrm{mid}}$, as in
RAUQ~\cite{vazhentsev2025uncertaintyawareattentionheadsefficient} (layers 10-22 for LLaMA-3.1-8B).
Within each such layer, we first select a routing head using a lightweight spectral proxy computed
only on the generated-token attention block. Let $B_T^{(l,h)}$ be the generated-token square block
at the final decoding step. We select
\begin{equation}
    h^*(l)
    =
    \arg\max_h
    \mathrm{PR}\!\left(
        B_T^{(l,h)} B_T^{(l,h)\top}
    \right).
    \label{eq:sauq-selection}
\end{equation}
This proxy is a cheap empirical spectral screen: it scores all $H$ heads of a layer with a single batched
product on $T\times T$ blocks, whereas the intrinsic signal below is computed only for the selected
head.
 
\paragraph{Spectral signal.}
For each selected head and at each generation step, we equip the full prefix with the intrinsic
measures of \eqref{eq:mixed-measures}, $\pi_t=\frac{1}{n_t}\mathbf 1$ and
$\mu_t=A_t^{(l,h^*)\top}\pi_t$, and compute the normalized participation ratio of the TDO
\begin{equation}
    s_t^{(l)} := \mathrm{PR}\!\left(D^{(l,h^*)}_{\mu_t,\pi_t}\right).
    \label{eq:sauq-signal}
\end{equation} 
The signal only requires $\operatorname{tr}(S_{\mu_t,\pi_t})$ and $\|S_{\mu_t,\pi_t}\|_F^2$, so no eigendecomposition is needed.
 
\paragraph{Confidence recurrence.}
Then, we integrate our signal into the recurrence structure of RAUQ~\cite{vazhentsev2025uncertaintyawareattentionheadsefficient}.
Let $p_t = P(y_t \mid y_{<t},x)$ be the model probability assigned to the generated token at step $t$.
SAUQ defines a layer-wise confidence recursively by
\begin{equation}
    c_t^{(l)}
    =
    \begin{cases}
        p_1, & t = 1, \\
        \alpha\, p_t + (1-\alpha)\, s_t^{(l)} c_{t-1}^{(l)}, & t > 1,
    \end{cases}
    \label{eq:sauq-recurrence}
\end{equation}
where $\alpha\in[0,1]$ trades off local token probability against confidence propagated through the
spectral routing structure; in the following experiments, we select $\alpha=0.2$ following
\cite{vazhentsev2025uncertaintyawareattentionheadsefficient}. The recurrence, $\alpha$, and the layer range are taken unchanged from RAUQ: they are not dictated by our theory, and keeping them fixed isolates the effect of the routing signal.
 
Under this recurrence, high spectral participation preserves prior confidence, reflecting the presence of multiple active routing directions. By contrast, spectral collapse attenuates propagated confidence, since the head routes information through fewer effective degrees of freedom. For each selected layer, we convert the confidence trajectory into an uncertainty score by averaging log-confidence over the generated sequence,
$U^{(l)}_{\mathrm{SAUQ}} = 1 - \frac{1}{T}\sum_{t=1}^T \log c_t^{(l)}$.
The final sequence-level score is the maximum uncertainty over the middle-third layers,
$U_{\mathrm{SAUQ}} = \max_{l \in \mathcal L_{\mathrm{mid}}} U^{(l)}_{\mathrm{SAUQ}}$.
We use the maximum because SAUQ is a rejection score: a routing collapse in any monitored head
should raise uncertainty rather than be averaged away. Since our analysis is per head and precedes
residual mixing and MLP updates, such a collapse need not imply information loss for the whole
model. Thus, SAUQ preserves the recurrent structure of RAUQ, but replaces its scalar previous-token
score with the effective spectral dimension of routing heads from Section~\ref{sec:theory}.
 
\paragraph{Generation-only variant.}
The full $D_{\mu_t,\pi_t}$ mixes prompt and generated tokens. To isolate the uncertainty accumulated
during the generation, we also consider a variant in which the signal is the same Euclidean proxy
used for head selection, evaluated at each step on the generated block,
$s_t^{(l),\mathrm{gen}} := \mathrm{PR}\big(B_t^{(l,h^*)} B_t^{(l,h^*)\top}\big)$, with selection,
recurrence and aggregation unchanged. Since the prompt columns are removed and $B_t$ is not
renormalized, $B_t$ is sub-stochastic: the intrinsic geometry of Section~\ref{sec:theory} does not
apply. In \ref{app:selection-grid}, we analyze the effect of having different selection-signal relations in SAUQ.

\subsection{Evaluation}
\label{subsec:sauq-experiments}

\begin{table*}[t]
\centering
\small
\resizebox{\textwidth}{!}{%
\begin{tabular}{lc|cccc|c|cc|c}
\toprule
& & \multicolumn{5}{c|}{\textbf{QA}} & \multicolumn{3}{c}{\textbf{Summarization}} \\
\cmidrule(lr){3-7} \cmidrule(lr){8-10}
\textbf{Estimator} & Category &\textbf{TriviaQA} & \textbf{TruthfulQA} & \textbf{MMLU} & \textbf{COQA} & \textbf{Mean} & \textbf{SAMSUM} & \textbf{XSUM} & \textbf{Mean} \\
\midrule
AttentionScore \cite{Sriramanan2024_attention_score} & Single pass & 0.0796 & -0.0505 & 0.2078 & 0.0933 & 0.0826 & -0.3727 & -0.9983 & -0.6855 \\
MaximumSequenceProbability & Single pass &0.2518 & 0.3211 & 0.5501 & 0.2878 & 0.3527 & -0.1363 & -0.386 & -0.2612 \\
MeanTokenEntropy & Single pass &0.5661 & \textbf{0.4199} & 0.2437 & 0.2568 & 0.3716 & 0.4999 & 0.3928 & 0.4464 \\
Perplexity & Single Pass &0.1088 & 0.2863 & 0.5158 & 0.2737 & 0.2962 & 0.5071 & \textcolor{blue}{0.4053} & \textcolor{blue}{0.4562} \\
CCP\cite{fadeeva-etal-2023-lm} & Single pass &0.1029 & 0.159 & 0.5567 & 0.2602 & 0.2697 & -0.1458 & -0.3018 & -0.2238 \\
NLI-DegMat  \cite{lin2024generatingconfidenceuncertaintyquantification} & Sample-based & 0.5789 & 0.2061 & 0.4381 & \textcolor{blue}{0.3959} & 0.4048 & 0.3888 & 0.3236 & 0.3562 \\
NLI-Eccentricity  \cite{lin2024generatingconfidenceuncertaintyquantification} & Sample-based & 0.5641 & 0.2007 & 0.4469 & 0.3805 & 0.3981 & 0.344 & 0.0351 & 0.1896 \\
NLI-EigValLaplacian \cite{lin2024generatingconfidenceuncertaintyquantification} & Sample-based & \textbf{0.6034} & 0.1962 & 0.4681 & \textbf{0.4147} & \textbf{0.4206} & 0.3864 & 0.3845 & 0.3855 \\
EigenScore \cite{chen2024insidellmsinternalstates} & Sample-based & 0.3611 & 0.0524 & 0.4118 & 0.3467 & 0.293 & 0.0719 & -0.2122 & -0.0702 \\
LUQ \cite{zhang2024luqlongtextuncertaintyquantification} & Sample-based & \textcolor{blue}{0.5878} & 0.182 & 0.4392 & 0.3588 & 0.392 & 0.29 & 0.3585 & 0.3243 \\
LexicalSimilarity\_rougeL \cite{fomicheva-etal-2020-unsupervised} & Sample-based & 0.4788 & 0.1599 & 0.4127 & 0.3636 & 0.3538 & 0.2483 & 0.1981 & 0.2232 \\
SAR \cite{duan2024shiftingattentionrelevancepredictive}  & Sample-based & 0.5733 & 0.1786 & 0.4103 & 0.3275 & 0.3724 & 0.2936 & 0.3069 & 0.3003 \\
SemanticDensity\cite{qiu2024semanticdensityuncertaintyquantification} & Sample-based & 0.5592 & 0.0755 & 0.1582 & 0.1102 & 0.2258 & 0.0752 & 0.1138 & 0.0945 \\
SemanticEntropy\cite{kuhn2023semanticuncertaintylinguisticinvariances} & Sample-based & 0.2005 & 0.2113 & 0.4701 & 0.3662 & 0.312 & 0.2525 & -0.1886 & 0.032 \\
RAUQ \cite{vazhentsev2025uncertaintyawareattentionheadsefficient} & Single pass & 0.2631 & 0.3369 & \textbf{0.5596} & 0.311 & 0.3677 & \textcolor{blue}{0.5269} & 0.3008 & 0.4139 \\
LapEigValRAUQ \cite{binkowski-etal-2025-hallucination} & Single pass & 0.1190 & 0.3296 & 0.5578 & 0.3064  & 0.3282 & 0.5285  & 0.2907	 &  0.4096\\
LogDetRAUQ \cite{Sriramanan2024_attention_score} & Single pass & 0.1140 & 0.3313 & 0.5596	  & 0.3087	  &  0.328 & 0.5288  &  0.3048 &  0.4168 \\
\midrule
\textbf{SAUQ (gen. only)}& Single pass & 0.5013 & 0.3056 & 0.5375 & 0.311 & \textcolor{blue}{0.4139}  & 0.4948 & 0.2578 &  0.3763 \\

\textbf{SAUQ (Ours)}& Single pass & 0.2914 & \textcolor{blue}{0.3767} & \textcolor{blue}{0.5588} & 0.3055 & 0.3831 & \textbf{0.5924} & \textbf{0.4214} & \textbf{0.5069} \\
\bottomrule
\end{tabular}%
}
\caption{Comparison against unsupervised methods. The best result in each column is shown in bold, and the second best in blue.}
\label{tab:estimators_comparison}
\end{table*}

We evaluate SAUQ using Llama-3.1-8b\cite{llama3modelcard} and Qwen2.5-7B~\cite{yang2024qwen2technicalreport} (Appendix~\ref{app:qwen}). We use the LM Polygraph benchmark~\cite{fadeeva-etal-2023-lm} to ensure easy reproducibility, following the details provided in the experimental setting of \cite{vazhentsev2025uncertaintyawareattentionheadsefficient}. We report the Normalized Prediction Rejection Ratio (PRR; higher is better). We evaluate our method on question-answering task on TriviaQA~\cite{joshi2017triviaqalargescaledistantly}, TruthfulQA~\cite{lin2022truthfulqameasuringmodelsmimic}, MMLU~\cite{hendrycks2021measuringmassivemultitasklanguage}, CoQA~\cite{reddy2019coqaconversationalquestionanswering}, and on the summarization task on SAMSum~\cite{Gliwa_2019_samsum}, and XSum~\cite{narayan2018xsum}. Table~\ref{tab:estimators_comparison} compares SAUQ performances against single-step probability-based methods (Perplexity, MaxSequenceProbability, MeanTokenEntropy, CCP), sample-based (10 samples) methods (NLI-based spectral measures, EigenScore, LUQ, LexicalSimilarity, SAR, Semantic Density, SemanticEntropy), and attention-based baselines (AttentionScore, RAUQ, LapEigvalRAUQ, LogDetRAUQ). The last two are variants of SAUQ in which the attention signal has been extracted based on LapEigval\cite{binkowski-etal-2025-hallucination} and AttentionScore\cite{Sriramanan2024_attention_score}.

In Table~\ref{tab:estimators_comparison}, we report the mean PRR values for Llama-3.1-8b; We refer to the Appendix for additional results.
The table shows a regime-dependent picture rather than uniform dominance. On QA datasets, SAUQ is competitive with strong unsupervised estimators, but it is not the best method on average; several sample-based baselines obtain stronger mean QA performance. In contrast, the gains are clearest on summarization, where SAUQ achieves the best performance on both SAMSum and XSum, and the strongest summarization mean among all methods considered.
This is consistent with our motivation: when uncertainty depends on preserving distinctions across long contexts, the upstream routing spectrum carries information that probability-based estimators have already lost. We therefore view SAUQ as a complementary single-pass signal rather than a universal replacement. 

Because SAUQ, RAUQ, LapEigvalRAUQ, and LogDetRAUQ share the same recurrence and protocol, their comparison primarily reflects the effect of replacing the scalar attention signal used. The comparison among them isolates the effect of the routing metric inside the estimator. The ranking mirrors the metric-quality findings of Section~\ref{sec:empirical}: LapEigval and LogDet spectral signals carry less of the routing signal than the intrinsic TDO, and this gap translates into stronger single-pass uncertainty.

\paragraph{Choice of measure}
\label{subsec:sauq-ablation}

The framework treats the input measures $\mu,\pi$ as an explicit modeling axis rather than a fixed convention. In the context of uncertainty, we tested alternative approaches to uniform and intrinsic measures, including masses based on token positions, frequency of tokens (Zipf's Law), and unembedding matrix decomposition. Definitions of all measures and an example of their distribution are given in Appendix~\ref{app:measures}.
\begin{wraptable}{r}{0.34\linewidth}
\vspace{-1.2em}
\centering
\small
\setlength{\tabcolsep}{3.5pt}
\renewcommand{\arraystretch}{0.9}
\begin{tabular}{lcc}
\toprule
\textbf{Measure} & \textbf{QA} & \textbf{Summ.} \\
\midrule
Uniform     & 0.3660 & 0.4204 \\
Positional  & \textbf{0.4027} & \textcolor{blue}{0.4771} \\
Zipf-Law    & 0.3718 & 0.4223 \\
Unembedding & 0.3755 & 0.4202 \\
Intrinsic   & \textcolor{blue}{0.3831} & \textbf{0.5069} \\
\bottomrule
\end{tabular}
\caption{Mean PRR of SAUQ under different measures (LLaMA-3.1-8B).}
\end{wraptable}
Only the intrinsic choice $\mu=A^\top\pi$ satisfies Proposition~\ref{thm:sigma1}; for the others the constant function is not an eigenvector of $D_{\mu,\pi}$.
On short-range QA, several measures perform comparably, with positional weighting marginally best. On summarization, where uncertainty depends more strongly on long-range routing, the attention-induced measure outperforms the alternatives, while counting geometry lags. 
Interestingly, the performances of frequency- and unembedding-based estimators suggest that the 'semantic' of the tokens is less relevant in this scenario with respect to positional bias. Extending such analysis forward offers an intriguing future direction in which our approach can be exploited for a broader understanding of LLMs' computations.

\section{Conclusion}
\label{sec:conclusions}

We introduced a new spectral-geometric framework for attention. Euclidean spectral diagnostics can be biased by sink behavior, conflating column-mass concentration with genuine routing; the intrinsic probability geometry instead yields a token-difference spectrum that separates the trivial constant mode from the centered modes controlling routing capacity. Empirically, TDO operator better disentangle relations between sink score and head-output dimensionality, and supports uncertainty estimation. Building on it, SAUQ replaces scalar attention probes with a spectral measure of routing capacity and improves over closely related attention-based baselines on long-range tasks.
\paragraph{Limitations.}
Our analysis is per-head, per-input, and single-step, characterizing the routing map before residual mixing, MLP updates, and cross-head output projection; extending it to the full dynamics of representation evolution is a natural next step. Practically, computing the full spectrum is more expensive than scalar attention statistics or the participation-ratio statistic used in our variant, requiring additional care to apply at scale.
\paragraph{Broader Impact.}
More faithful diagnostics of attention routing may improve interpretability, uncertainty estimation, and reliability assessment in transformer models, especially for deployment of safe and reliable AI systems.

\begin{ack}
This work was supported by the ERC Consolidator Grant 101087347 (VEGA), the ANR AI Chair AIGRETTE, and by gifts from Ansys and Adobe Research. The MUR supported this work with the grant "Dipartimenti di Eccellenza 2023-2027" REGAINS of the Department of Informatics, Systems and Communication of the University of Milano-Bicocca. 
\end{ack}

\bibliographystyle{plain}
\bibliography{bib}
\newpage
\appendix
\textbf{Supplementary Material/appendix for "Spectral Geometry of Attention:
From Information Routing to Uncertainty".}

\section{Proofs for the Theoretical Section}
\subsection{Full Notation}


We consider a decoder-only transformer with $L$ layers and $H$ attention heads per layer, processing a sequence of $n$ tokens. At layer $l$, the input is $X_{l-1} \in \mathbb{R}^{n \times d}$, whose $i$-th row $x_i$ is the hidden representation of token $i$. We adopt the pre-norm convention as in~\cite{barbero2025glasses}:
\begin{equation}
\begin{aligned}
    \hat{X}_{l-1} = \mathrm{LN}(X_{l-1}),&
    M_l = X_{l-1} + \mathrm{MHA}_l(\hat{X}_{l-1}),& 
    X_l = M_l + \mathrm{MLP}_l(\mathrm{LN}(M_l)).
\end{aligned}
\label{eq:transformer-block}
\end{equation}
For each head $h$, queries, keys, and values are $Q_{l,h}=\hat{X}_{l-1}W^Q_{l,h}$, $ 
K_{l,h}=\hat{X}_{l-1}W^K_{l,h}$, $
V_{l,h}=\hat{X}_{l-1}W^V_{l,h}
$. The attention matrix is
\begin{equation}
    A_{l,h}
    =
    \mathrm{softmax}\!\left(Q_{l,h}K_{l,h}^\top/\sqrt{d} + C\right),
    \label{eq:attn2}
\end{equation}
where $C$ is the causal mask, with $C_{ij}=0$ for $j\le i$ and $C_{ij}=-\infty$ otherwise. The full multi-head attention output is obtained by concatenating the head and projecting outputs $
    \mathrm{MHA}_l(\hat{X}_{l-1})
    =
    \text{concat}_{h}\!\left(A_{l,h}V_{l,h}\right)W^O_l .$

\label{app:theory}

In this section we provide the technical details for the results stated in
Section~\ref{sec:theory}. 

\subsection{TDO is self-adjoint}
\label{app:weighted-adjoint-tdo}
We recall $
    \mathcal A:L^2(\mu)\longrightarrow L^2(\pi),
    \quad
    \mathcal A f = Af.
$. Its adjoint is defined as the operator $A^*$  such that $ \langle \mathcal A f,g\rangle_\pi
    =
    \langle f,\mathcal A^*g\rangle_\mu  $.
We compute it explicitly
\begin{equation}
    \langle \mathcal A f,g\rangle_\pi
    =
    (Af)^\top \Omega_\pi g                                           
    =
    f^\top A^\top \Omega_\pi g                                       
    =
    f^\top \Omega_\mu \bigl(\Omega_\mu^{-1}A^\top \Omega_\pi g\bigr)            
    =
    \langle f, \Omega_\mu^{-1}A^\top \Omega_\pi g\rangle_\mu .
\end{equation}
Hence, $    \mathcal A^*
    =
    \Omega_\mu^{-1}A^\top \Omega_\pi .
$
The token-difference operator $    D_{\mu,\pi}
    :=
    \mathcal A^*\mathcal A
    =
    \Omega_\mu^{-1}A^\top \Omega_\pi A .$ is self-adjoint in $L^2(\mu)$. Indeed, for any $f,g\in L^2(\mu)$,
\begin{equation}
    \langle D_{\mu,\pi}f,g\rangle_\mu
    =
    \langle \mathcal A^*\mathcal A f,g\rangle_\mu                 
    \langle \mathcal A f,\mathcal A g\rangle_\pi                  
    =
    \langle f,\mathcal A^*\mathcal A g\rangle_\mu                 
    =
    \langle f,D_{\mu,\pi}g\rangle_\mu .
\end{equation}
It is also positive semidefinite:
\begin{equation}
    \langle D_{\mu,\pi}f,f\rangle_\mu
    =
    \langle \mathcal A f,\mathcal A f\rangle_\pi
    =
    \|Af\|_\pi^2
    \ge 0 .
\end{equation}

The associated Rayleigh quotient is
\begin{equation}
    \mathcal R_{\mu,\pi}(f)
    =
    \frac{\|Af\|_\pi^2}{\|f\|_\mu^2}
    =
    \frac{f^\top A^\top \Omega_\pi A f}{f^\top \Omega_\mu f}
    =
    \frac{\langle D_{\mu,\pi}f,f\rangle_\mu}{\langle f,f\rangle_\mu}.
\end{equation}
Thus the eigenvalues of $D_{\mu,\pi}$ describe the squared gains of attention under the input geometry $\mu$ and the output geometry $\pi$.

Although $D_{\mu,\pi}$ is not generally symmetric in the Euclidean inner product,
it is similar to the symmetric positive semidefinite matrix
\begin{equation}
    S_{\mu,\pi}
    :=
    \Omega_\mu^{-1/2}A^\top \Omega_\pi A \Omega_\mu^{-1/2}.
\end{equation}
Indeed,
\begin{equation}
    D_{\mu,\pi}
    =
    \Omega_\mu^{-1/2}S_{\mu,\pi}\Omega_\mu^{1/2}.
\end{equation}
Therefore $D_{\mu,\pi}$ and $S_{\mu,\pi}$ have the same eigenvalues. In practice, we can compute spectral quantities using $S_{\mu,\pi}$.

\subsection{Euclidean case: proof of Proposition~\ref{prop:counting-noncontractive}}
\label{app:counting}

When both $\mu$ and $\pi$ are counting measures, $\Omega_\mu=\Omega_\pi=I$ and $    D_{\mu,\pi}
    =
    A^\top A .$
Thus, the eigenvalues of $D_{\mu,\pi}$ are the squared singular values of $A$.

\begin{proof}[Proof of Proposition~\ref{prop:counting-noncontractive}]
Because $A$ is row-stochastic, $A\mathbf 1=\mathbf 1$. Therefore
$    \|A\mathbf 1\|_2
    =
    \|\mathbf 1\|_2 .
$ By the definition of the spectral norm,
\begin{equation}
    \sigma_1(A)
    =
    \|A\|_2
    =
    \sup_{f\ne 0}\frac{\|Af\|_2}{\|f\|_2}
    \ge
    \frac{\|A\mathbf 1\|_2}{\|\mathbf 1\|_2}
    =
    1 .
\end{equation}

Now suppose that $\sigma_1(A)=1$. Then $\|A\|_2=1$, and the vector
$\mathbf 1$ attains the operator norm. Equivalently, the Rayleigh quotient of
$A^\top A$ at $\mathbf 1$ equals its largest eigenvalue:
\begin{equation}
    \frac{\mathbf 1^\top A^\top A\mathbf 1}
         {\mathbf 1^\top \mathbf 1}
    =
    1
    =
    \lambda_{\max}(A^\top A).
\end{equation}
Since $A^\top A$ is symmetric positive semidefinite, equality in the Rayleigh
principle implies that $\mathbf 1$ is an eigenvector of $A^\top A$ with
eigenvalue $1$:
$    A^\top A\mathbf 1
    =
    \mathbf 1 . $
Using $A\mathbf 1=\mathbf 1$, this becomes$A^\top \mathbf 1
    =
    \mathbf 1 .$ Hence every column of $A$ also sums to one, so $A$ is doubly stochastic.

Conversely, suppose that $A$ is doubly stochastic. Since $A$ is nonnegative and
has row sums and column sums equal to one,
\begin{equation}
    \|A\|_\infty = 1,
    \qquad
    \|A\|_1 = 1 .
\end{equation}
The standard norm interpolation bound gives
\begin{equation}
    \|A\|_2
    \le
    \sqrt{\|A\|_1\|A\|_\infty}
    =
    1 .
\end{equation}
Together with the lower bound $\|A\|_2\ge 1$, this gives
$\sigma_1(A)=\|A\|_2=1$.
\end{proof}

\paragraph{Frobenius budget.}
The total squared singular mass is
\begin{equation}
    \sum_k \sigma_k(A)^2
    =
    \|A\|_F^2
    =
    \sum_{i,j}A_{ij}^2 .
\end{equation}
Since each row of $A$ is a probability vector, $0\le A_{ij}\le 1$, and therefore
$A_{ij}^2\le A_{ij}$. Hence
\begin{equation}
    \|A\|_F^2
    =
    \sum_{i,j}A_{ij}^2
    \le
    \sum_{i,j}A_{ij}
    =
    n .
\end{equation}
Consequently,
\begin{equation}
    \sum_{k\ge 2}\sigma_k(A)^2
    =
    \|A\|_F^2-\sigma_1(A)^2
    \le
    n-\sigma_1(A)^2 .
\end{equation}
Thus, when the leading singular value is large, a large fraction of the available Frobenius budget is concentrated in the leading mode. 

There is also a lower bound. For any probability vector $p$ supported on $s$ coordinates, $    \sum_{j=1}^s p_j^2
    \ge
    \frac{1}{s},
$
by Cauchy's inequality. Without causal constraints, every row may use all $n$
positions, so $\|A\|_F^2\ge 1$, with equality at the fully uniform matrix.
Under causal masking, row $i$ is supported on at most $i$ positions. Therefore $    \sum_j A_{ij}^2
    \ge
    \frac{1}{i}, $ and hence
$    \|A\|_F^2
    \ge
    \sum_{i=1}^n \frac{1}{i}.
$

\subsection{Intrinsic probability geometry: proof of Proposition~\ref{thm:sigma1}}
\label{app:thm-sigma1}

We now use the intrinsic measures
\begin{equation}
    \pi=\frac1n\mathbf 1,
    \qquad
    \mu=A^\top \pi .
\end{equation}

\begin{proof}[Proof of Proposition~\ref{thm:sigma1}]
For each query position $i$, the row $A_{i:}$ is a probability vector. Therefore,
by Jensen's inequality applied to the convex function $x\mapsto x^2$,
\begin{equation}
    (Af)_i^2
    =
    \left(\sum_j A_{ij}f_j\right)^2
    \le
    \sum_j A_{ij}f_j^2 .
\end{equation}
Averaging this inequality over $i$ with the uniform measure $\pi_i=1/n$ gives
\begin{equation}
    \|Af\|_\pi^2
    =
    \frac1n\sum_{i=1}^n (Af)_i^2                                      
    \le
    \frac1n\sum_{i=1}^n\sum_{j=1}^n A_{ij}f_j^2                        
    =
    \sum_{j=1}^n
    \left(\frac1n\sum_{i=1}^n A_{ij}\right)f_j^2                       
    =
    \sum_{j=1}^n \mu_j f_j^2                                           =
    \|f\|_\mu^2 .
\end{equation}
Hence
$    \|\mathcal A\|_{L^2(\mu)\to L^2(\pi)}
    \le
    1 .
$ Since $A$ is row-stochastic $    A\mathbf 1=\mathbf 1. $
Because both $\mu$ and $\pi$ are probability measures their respective norm is 1. Therefore
$\|\mathcal A\mathbf 1\|_\pi
    =
    \|\mathbf 1\|_\pi
    =
    \|\mathbf 1\|_\mu .
$ So the operator norm is exactly one.

Finally, the eigenvalues of $D_{\mu,\pi}$ are squared singular gains of
$\mathcal A:L^2(\mu)\to L^2(\pi)$. The previous inequality shows that all eigenvalues are at most one, and the constant function attains gain one. Hence $\lambda_1(D_{\mu,\pi})=1,
    \qquad
    D_{\mu,\pi}\mathbf 1=\mathbf 1 .$ \end{proof}

The last identity can also be checked directly:
\begin{equation}
    D_{\mu,\pi}\mathbf 1
    =
    \Omega_\mu^{-1}A^\top \Omega_\pi A\mathbf 1       
    =
    \Omega_\mu^{-1}A^\top \Omega_\pi \mathbf 1           =
    \Omega_\mu^{-1}A^\top \pi                    
    =
    \Omega_\mu^{-1}\mu                           
    =
    \mathbf 1 .
\end{equation}
Moreover,
\begin{equation}
    \mu^\top D_{\mu,\pi}
    =
    \mu^\top \Omega_\mu^{-1}A^\top \Omega_\pi A        
    =
    \mathbf 1^\top A^\top \Omega_\pi A            
    =
    (A\mathbf 1)^\top \Omega_\pi A                
    =
    \mathbf 1^\top \Omega_\pi A                   
    =
    \pi^\top A                               
    =
    \mu^\top .
\end{equation}
Thus $\mathbf 1$ is a right eigenvector and $\mu$ is the corresponding invariant left measure.

\subsection{Proof of Corollary~\ref{cor:centered-contraction}}
\label{app:cor-centered}

Define the $\mu$-centered subspace
\begin{equation}
    \mathbf 1^{\perp_\mu}
    =
    \{f\in L^2(\mu):\mu^\top f=0\}.
\end{equation}
The identity $\pi^\top A=\mu^\top$ implies that attention maps centered functions
to centered functions. Indeed, if $f\in \mathbf 1^{\perp_\mu}$, then
\begin{equation}
    \pi^\top Af
    =
    \mu^\top f
    =
    0 .
\end{equation}
Thus $Af$ is $\pi$-centered.

Next, since $D_{\mu,\pi}$ is self-adjoint in $L^2(\mu)$ and
$D_{\mu,\pi}\mathbf 1=\mathbf 1$, the orthogonal complement
$\mathbf 1^{\perp_\mu}$ is invariant under $D_{\mu,\pi}$. To see this directly,
take any $f\in \mathbf 1^{\perp_\mu}$. Then
\begin{equation}
    \langle D_{\mu,\pi}f,\mathbf 1\rangle_\mu
    =
    \langle f,D_{\mu,\pi}\mathbf 1\rangle_\mu
    =
    \langle f,\mathbf 1\rangle_\mu
    =
    \mu^\top f
    =
    0 .
\end{equation}
Hence $D_{\mu,\pi}f\in \mathbf 1^{\perp_\mu}$.

For $f\in \mathbf 1^{\perp_\mu}$, the squared gain of attention is
\begin{equation}
    \frac{\|Af\|_\pi^2}{\|f\|_\mu^2}
    =
    \frac{\langle D_{\mu,\pi}f,f\rangle_\mu}{\langle f,f\rangle_\mu}.
\end{equation}
By the Rayleigh-Ritz variational principle applied to the restriction of
$D_{\mu,\pi}$ to $\mathbf 1^{\perp_\mu}$,
\begin{equation}
    \sup_{f\in \mathbf 1^{\perp_\mu}\setminus\{0\}}
    \frac{\|Af\|_\pi^2}{\|f\|_\mu^2}
    =
    \lambda_{\max}\!\left(D_{\mu,\pi}\big|_{\mathbf 1^{\perp_\mu}}\right).
\end{equation}
With the eigenvalues ordered as
\begin{equation}
    1=\lambda_1\ge \lambda_2\ge \cdots\ge 0,
\end{equation}
and with one copy of the eigenvalue $\lambda_1=1$ corresponding to the constant
function, this restricted maximum is $\lambda_2$. Therefore
\begin{equation}
    \sup_{f\in \mathbf 1^{\perp_\mu}\setminus\{0\}}
    \frac{\|Af\|_\pi}{\|f\|_\mu}
    =
    \sqrt{\lambda_2(D_{\mu,\pi})}
    \le
    1 .
\end{equation}
If the unit eigenvalue has multiplicity larger than one, then there exist
nonconstant centered modes with unit gain, and the same formula gives
$\lambda_2(D_{\mu,\pi})=1$.

\subsection{Proof of Proposition~\ref{thm:mean-routing-decomposition}}
\label{app:thm-mean-routing}

For matrix-valued token signals $U,W\in\mathbb R^{n\times d}$, define
\begin{equation}
    \langle U,W\rangle_{\mu,F}
    =
    \operatorname{tr}(U^\top \Omega_\mu W),
    \qquad
    \|U\|_{\mu,F}^2
    =
    \operatorname{tr}(U^\top \Omega_\mu U).
\end{equation}
The analogous $\pi$-weighted quantities are defined by replacing $\Omega_\mu$ with
$\Omega_\pi$.

Given a value matrix $V\in\mathbb R^{n\times d}$, define its $\mu$-mean and
centered component by
\begin{equation}
    m:=\mu^\top V\in\mathbb R^d,
    \qquad
    V_c:=V-\mathbf 1m .
\end{equation}
Then
\begin{equation}
    \mu^\top V_c
    =
    \mu^\top V-\mu^\top\mathbf 1m
    =
    m-(\mathbf 1^\top\mu)m
    =
    m-m
    =
    0,
\end{equation}
because $\mu$ is a probability measure. Thus every feature column of $V_c$ lies
in $\mathbf 1^{\perp_\mu}$.

Using $V=\mathbf 1m+V_c$ and $A\mathbf 1=\mathbf 1$, we obtain
\begin{equation}
    AV
    =
    A(\mathbf 1m+V_c)
    =
    \mathbf 1m+AV_c .
\end{equation}
This proves the additive decomposition.

We now prove the orthogonality of the two terms in $L^2(\pi)$. Since
$\pi^\top A=\mu^\top$ and $\mu^\top V_c=0$,
\begin{equation}
    \langle \mathbf 1m,AV_c\rangle_{\pi,F}
    =
    \operatorname{tr}\!\left((\mathbf 1m)^\top \Omega_\pi AV_c\right) 
    =
    \operatorname{tr}\!\left(m^\top \mathbf 1^\top \Omega_\pi AV_c\right) 
    =
    \operatorname{tr}\!\left(m^\top \pi^\top AV_c\right) 
    =
    \operatorname{tr}\!\left(m^\top \mu^\top V_c\right) 
    =
    0 .
\end{equation}

Finally, write $V_c=[v_c^{(1)},\ldots,v_c^{(d)}]$, where each
$v_c^{(r)}\in\mathbf 1^{\perp_\mu}$. Applying
Corollary~\ref{cor:centered-contraction} to each feature column gives
\begin{equation}
    \|A v_c^{(r)}\|_\pi^2
    \le
    \lambda_2(D_{\mu,\pi})\|v_c^{(r)}\|_\mu^2 .
\end{equation}
Summing over feature coordinates,
\begin{equation}
\begin{aligned}
    \|AV_c\|_{\pi,F}^2
    =
    \sum_{r=1}^d \|A v_c^{(r)}\|_\pi^2 
    \le
    \lambda_2(D_{\mu,\pi})
    \sum_{r=1}^d \|v_c^{(r)}\|_\mu^2 
    =
    \lambda_2(D_{\mu,\pi})\|V_c\|_{\mu,F}^2 .
\end{aligned}
\end{equation}
Taking square roots gives
\begin{equation}
    \|AV_c\|_{\pi,F}
    \le
    \sqrt{\lambda_2(D_{\mu,\pi})}\,\|V_c\|_{\mu,F}.
\end{equation}

\subsection{Proof of Corollary~\ref{cor:spectral-entropy-routed}}
\label{app:cor-spectral-entropy}

Let
\begin{equation}
    Z=\Omega_\pi^{1/2}AV,
    \qquad
    Z_c=\Omega_\pi^{1/2}AV_c .
\end{equation}
By Proposition~\ref{thm:mean-routing-decomposition},
\begin{equation}
    AV=\mathbf 1m+AV_c .
\end{equation}
Multiplying by $\Omega_\pi^{1/2}$ gives
\begin{equation}
    Z
    =
    \Omega_\pi^{1/2}\mathbf 1m+Z_c .
\end{equation}
Define
\begin{equation}
    Z_m:=\Omega_\pi^{1/2}\mathbf 1m .
\end{equation}
Since $\pi$ is a probability measure,
\begin{equation}
    Z_m^\top Z_m
    =
    m^\top \mathbf 1^\top \Omega_\pi \mathbf 1 m
    =
    m^\top m .
\end{equation}
Moreover, the cross term vanishes as a matrix:
\begin{equation}
    Z_m^\top Z_c
    =
    m^\top \mathbf 1^\top \Omega_\pi A V_c 
    =
    m^\top \pi^\top A V_c 
    =
    m^\top \mu^\top V_c 
    =
    0 .
\end{equation}
Therefore
\begin{equation}
    Z^\top Z
    =
    Z_m^\top Z_m+Z_c^\top Z_c
    =
    m^\top m+Z_c^\top Z_c ,
\end{equation}
and
\begin{equation}
    \|Z\|_F^2
    =
    \|m\|_2^2+\|Z_c\|_F^2 .
\end{equation}

Assume first that both $m\ne 0$ and $Z_c\ne 0$. Define the normalized spectral
densities
\begin{equation}
    \rho_Z
    =
    \frac{Z^\top Z}{\operatorname{tr}(Z^\top Z)},
    \qquad
    \rho_m
    =
    \frac{m^\top m}{\|m\|_2^2},
    \qquad
    \rho_c
    =
    \frac{Z_c^\top Z_c}{\|Z_c\|_F^2}.
\end{equation}
Let
\begin{equation}
    \alpha
    =
    \frac{\|Z_c\|_F^2}
         {\|m\|_2^2+\|Z_c\|_F^2}.
\end{equation}
Then
\begin{equation}
    \rho_Z
    =
    (1-\alpha)\rho_m+\alpha\rho_c .
\end{equation}
The matrix $\rho_m$ has rank one, so its entropy is zero:
\begin{equation}
    H(\rho_m)=0 .
\end{equation}
Using concavity of entropy,
\begin{equation}
    H(\rho_Z)
    \ge
    (1-\alpha)H(\rho_m)+\alpha H(\rho_c)
    =
    \alpha H(\rho_c).
\end{equation}
Since $H(\rho_Z)=H(Z)$ and $H(\rho_c)=H(Z_c)$, this gives
\begin{equation}
    H(Z)
    \ge
    \alpha H(Z_c).
\end{equation}

For the upper bound, we use the standard entropy bound for a two-component
mixture:
\begin{equation}
    H\bigl((1-\alpha)\rho_m+\alpha\rho_c\bigr)
    \le
    h_2(\alpha)
    +(1-\alpha)H(\rho_m)
    +\alpha H(\rho_c),
\end{equation}
where
\begin{equation}
    h_2(\alpha)
    =
    -\alpha\log\alpha-(1-\alpha)\log(1-\alpha).
\end{equation}
Since $H(\rho_m)=0$, this yields
\begin{equation}
    H(Z)
    \le
    h_2(\alpha)+\alpha H(Z_c).
\end{equation}
The cases $m=0$ or $Z_c=0$ follow by continuity, with the convention
$\alpha H(Z_c)=0$ when $Z_c=0$.

Finally,
\begin{equation}
    \|Z_c\|_F
    =
    \|\Omega_\pi^{1/2}AV_c\|_F
    =
    \|AV_c\|_{\pi,F}.
\end{equation}
The routed contraction bound from Proposition~\ref{thm:mean-routing-decomposition}
therefore gives
\begin{equation}
    \|Z_c\|_F
    \le
    \sqrt{\lambda_2(D_{\mu,\pi})}\,\|V_c\|_{\mu,F}.
\end{equation}

\subsection{Centered participation ratio}
\label{app:participation-ratio}
To isolate the nontrivial routed directions, we may also define the centered participation ratio
\begin{equation}
    \mathrm{PR}_{\mathrm{cent}}(D_{\mu,\pi})
    =
    \frac{1}{n-1}
    \frac{\left(\sum_{k=2}^n \lambda_k\right)^2}
         {\sum_{k=2}^n \lambda_k^2},
\end{equation}
whenever $\sum_{k=2}^n\lambda_k^2>0$. If all centered eigenvalues vanish, we set
$\mathrm{PR}_{\mathrm{cent}}(D_{\mu,\pi})=0$.

\section{Additional Results}

\subsection{Correlation on other models.}\begin{figure}[!ht]
    \centering

    \textbf{GPT-2}
    \includegraphics[width=\linewidth]{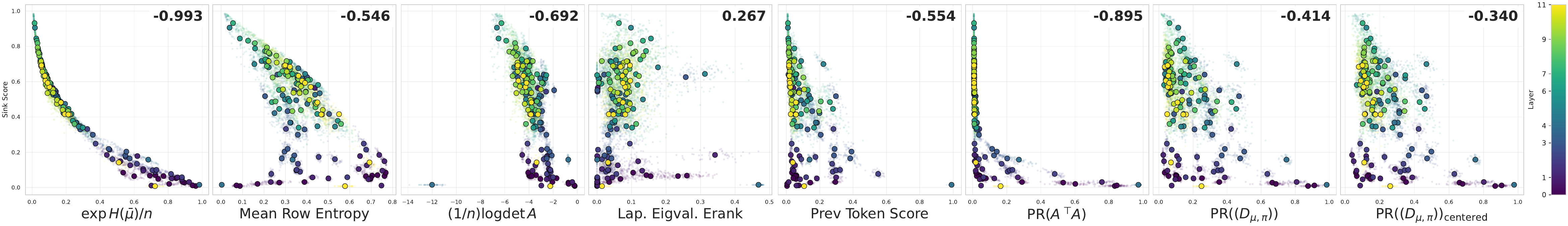}
    \includegraphics[width=\linewidth]{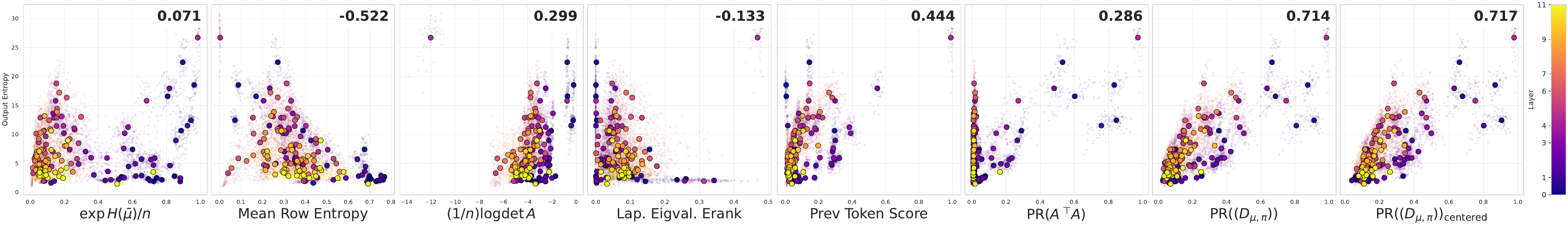}
    \textbf{Pythia-410m}
    \includegraphics[width=\linewidth]{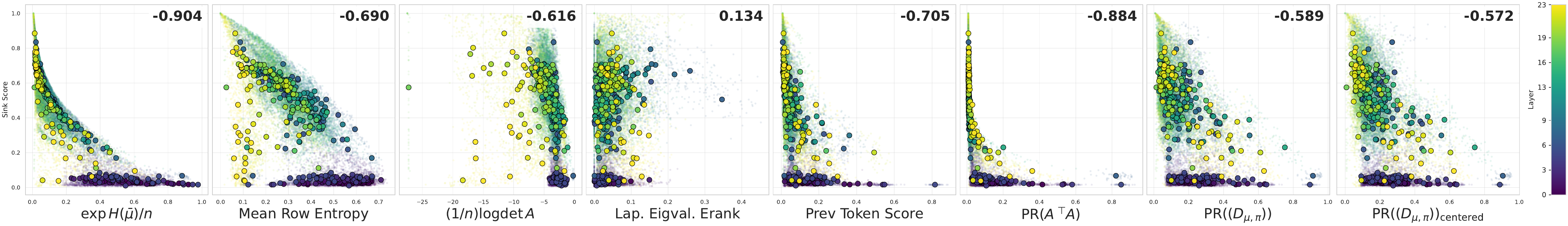}
    \includegraphics[width=\linewidth]{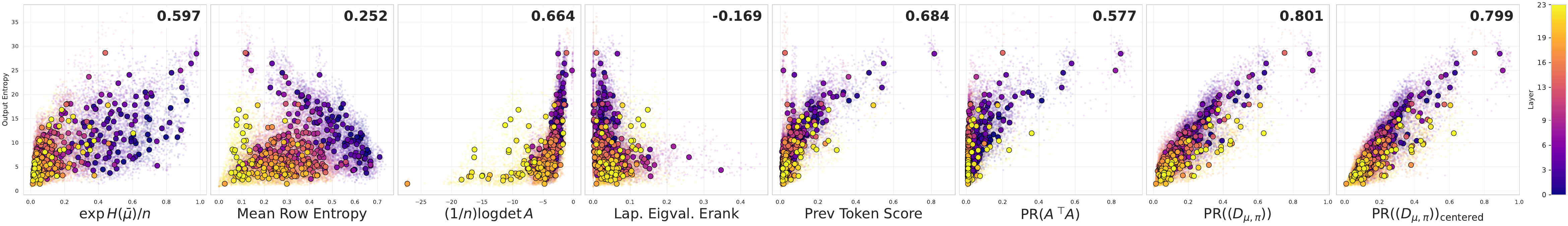}
    \textbf{Qwen2-7b}
    \includegraphics[width=\linewidth]{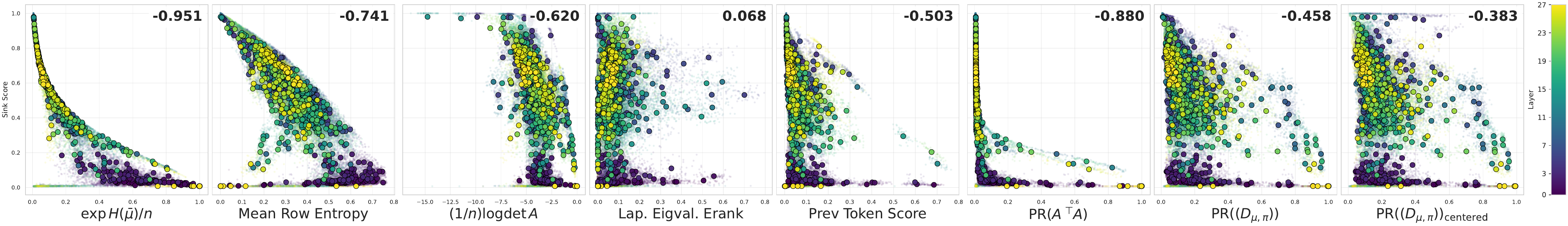}
    \includegraphics[width=\linewidth]{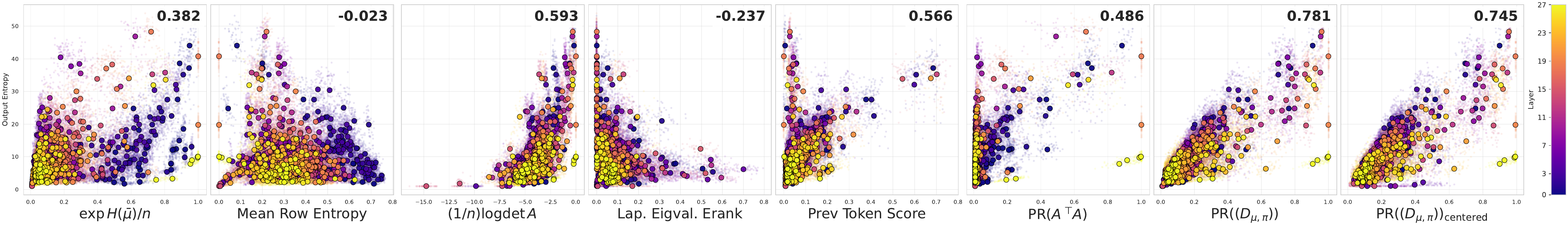}
    \caption{ Scatterplot between attention diagnostics, sink score, and output dimensionality on GPT-2, Pythia-410m and Qwen2-7b}
    \label{fig:correlation_new}
\end{figure}

We report the version of Figure \ref{fig:correlation-panels} but using three different models, GPT-2, Pythia-410m \cite{biderman2023pythiasuiteanalyzinglarge} and Qwen2-7b \cite{yang2024qwen2technicalreport}. The correlation plots are provided in Figure \ref{fig:correlation_new}. We note that the relation between the metrics seems independent of the model, being based on structural properties common to all attention maps.

\subsection{Supervised hallucination detection with spectral probes}\label{sec:supervised-hallucination}
 
Beyond unsupervised SAUQ, we probe the token-difference spectrum in the supervised
hallucination-detection protocol of LapEigval~\cite{binkowski-etal-2025-hallucination}, which trains a classifier on
eigenvalue features of the attention-induced graph Laplacian. We replace those features
with the leading $k=16$ eigenvalues of $D_{\mu,\pi}$ for each selected head. 
Truncation is necessary because, unlike the Laplacian, whose diagonal already supplies cheap
summaries, $D_{\mu,\pi}$ is dense and weighted, so its full spectrum is comparatively expensive
to estimate.

\begin{table}[!ht]
\centering
\caption{
Supervised hallucination-detection performance (ROC-AUC). We compare attention-matrix baselines, Laplacian eigenvalues, and features derived from the weighted token-difference operator. We follow the experimental protocol of~\cite{binkowski-etal-2025-hallucination} using Llama-3.1-8b \cite{llama3modelcard}, DeepSeek-4 as a judge.  
}
\label{tab:supervised-hallucination}
\begin{tabular}{lcccc}
\toprule
Dataset & AttnEigvals & AttnLogDet & LapEigvals  & TokenDiffEigvals \\
\midrule
COQA        & 0.7995 & 0.7779 & \textbf{0.8519} & \textcolor{blue}{0.8241} \\
GSM8K       & 0.8598 & 0.8586 & \textcolor{blue}{0.8660}  & \textbf{0.8663} \\
TriviaQA    & 0.8385 & 0.8346 & \textbf{0.8923}  & \textcolor{blue}{0.8731} \\
TruthfulQA  & 0.7997 & \textcolor{blue}{0.8362} & 0.8095  & \textbf{0.8445} \\
\bottomrule
\end{tabular}
\end{table}
 
Table~\ref{tab:supervised-hallucination} shows that TDO eigenvalues are competitive with
Laplacian features. 
In the low-$k$ regime, the leading eigenvalues of $D_{\mu,\pi}$ often match or exceed the
Laplacian features, suggesting that the dominant routing modes carry hallucination-relevant
information. We do not claim uniform dominance: with richer or cheaper Laplacian features,
the baseline can match or surpass the probe. Together with the unsupervised SAUQ results,
this suggests that the $D_{\mu,\pi}$ spectrum carries uncertainty information not reducible
to scalar attention statistics.

\subsection{Uncertainty quantification on Qwen2.5-7B}
\label{app:qwen}
 
Table~\ref{tab:estimators_comparison_2} replicates the evaluation of Section~\ref{sec:uq} on
Qwen2.5-7B~\cite{yang2024qwen2technicalreport}, using the same protocol, the same $\alpha=0.2$
and the middle third of layers.
 
The results confirm the regime dependence of Section~\ref{sec:uq}, but they also show that
the gains are model dependent. SAUQ improves over RAUQ on summarization
($0.3577$ vs.\ $0.3324$) and obtains the best MMLU score ($0.7204$), but it is below RAUQ on the
QA mean ($0.3724$ vs.\ $0.3944$). Unlike on LLaMA-3.1-8B, probability-based estimators
(MeanTokenEntropy, Perplexity) are the strongest on summarization, and the generated-only
variant outperforms the full-context one ($0.3654$ vs.\ $0.3577$). We interpret
the smaller advantage of the full-context signal as consistent with weaker sink behavior in
this model, since the intrinsic measure primarily corrects sink-induced bias; we list the
dependence on architecture as a limitation in Section~\ref{sec:conclusions}.
 
\begin{table*}[t]
\centering
\small
\resizebox{\textwidth}{!}{%
\begin{tabular}{lc|cccc|c|cc|c}
\toprule
 & & \multicolumn{5}{c|}{\textbf{QA}} & \multicolumn{3}{c}{\textbf{Summarization}} \\
\cmidrule(lr){3-7} \cmidrule(lr){8-10}
\textbf{Estimator} & Category & \textbf{TriviaQA} & \textbf{TruthfulQA} & \textbf{MMLU} & \textbf{CoQA} & \textbf{Mean} & \textbf{SAMSum} & \textbf{XSum} & \textbf{Mean} \\
\midrule
AttentionScore \cite{Sriramanan2024_attention_score} & Single pass & 0.0429 & 0.0466 & 0.1114 & 0.0022 & 0.0508 & 0.2924 & -0.0540 & 0.1192 \\
AttentionScore (gen.\ only) \cite{Sriramanan2024_attention_score} & Single pass & 0.1529 & -0.1053 & \textcolor{blue}{0.6741} & 0.0892 & 0.2027 & -0.0726 & -0.5038 & -0.2882 \\
MaximumSequenceProbability & Single pass & 0.2328 & 0.0436 & 0.6644 & 0.3016 & 0.3106 & 0.1220 & -0.0610 & 0.0305 \\
MeanTokenEntropy & Single pass & 0.6173 & 0.2474 & 0.2491 & 0.2533 & 0.3418 & \textcolor{blue}{0.4265} & 0.3363 & \textbf{0.3814} \\
Perplexity & Single pass & -0.0126 & 0.2093 & 0.4042 & 0.2620 & 0.2157 & \textbf{0.4303} & 0.3309 & \textcolor{blue}{0.3806} \\
CCP \cite{fadeeva-etal-2023-lm} & Single pass & 0.0636 & -0.0244 & 0.4957 & 0.2334 & 0.1921 & 0.0975 & 0.1419 & 0.1197 \\
NLI-DegMat \cite{lin2024generatingconfidenceuncertaintyquantification} & Sample-based & \textbf{0.6393} & 0.3117 & 0.3758 & \textbf{0.4073} & \textcolor{blue}{0.4335} & 0.2956 & 0.3367 & 0.3162 \\
NLI-Eccentricity \cite{lin2024generatingconfidenceuncertaintyquantification} & Sample-based & 0.6149 & 0.2472 & 0.3318 & 0.3476 & 0.3854 & 0.2883 & 0.2600 & 0.2742 \\
NLI-EigValLaplacian \cite{lin2024generatingconfidenceuncertaintyquantification} & Sample-based & 0.6201 & \textcolor{blue}{0.3147} & 0.3213 & 0.3855 & 0.4104 & 0.2934 & \textcolor{blue}{0.3535} & 0.3235 \\
EigenScore \cite{chen2024insidellmsinternalstates} & Sample-based & 0.4990 & -0.0506 & 0.3184 & 0.3282 & 0.2738 & 0.1783 & 0.0123 & 0.0953 \\
LUQ \cite{zhang2024luqlongtextuncertaintyquantification} & Sample-based & \textcolor{blue}{0.6334} & \textbf{0.3777} & 0.3576 & \textcolor{blue}{0.3900} & \textbf{0.4397} & 0.2697 & \textbf{0.4477} & 0.3587 \\
LexicalSimilarity\_rougeL \cite{fomicheva-etal-2020-unsupervised} & Sample-based & 0.6001 & 0.1317 & 0.3628 & 0.3697 & 0.3661 & 0.3237 & 0.2437 & 0.2837 \\
SAR \cite{duan2024shiftingattentionrelevancepredictive} & Sample-based & 0.6116 & 0.2120 & 0.4401 & 0.3446 & 0.4021 & 0.2869 & 0.3027 & 0.2948 \\
SemanticDensity \cite{qiu2024semanticdensityuncertaintyquantification} & Sample-based & 0.4903 & 0.1312 & 0.1323 & 0.0731 & 0.2067 & 0.1083 & 0.1898 & 0.1491 \\
SemanticEntropy \cite{kuhn2023semanticuncertaintylinguisticinvariances} & Sample-based & 0.3098 & 0.1179 & 0.3018 & 0.3066 & 0.2590 & 0.2640 & 0.0225 & 0.1433 \\
RAUQ \cite{vazhentsev2025uncertaintyawareattentionheadsefficient} & Single pass & 0.5226 & 0.1970 & 0.5939 & 0.2639 & 0.3944 & 0.3898 & 0.2750 & 0.3324 \\
\midrule
\textbf{SAUQ (gen.\ only)} & Single pass & 0.5906 & 0.1290 & 0.5182 & 0.2953 & 0.3833 & 0.3824 & 0.3484 & 0.3654 \\
\textbf{SAUQ (Ours)} & Single pass & 0.3780 & 0.1804 & \textbf{0.7204} & 0.2107 & 0.3724 & 0.3988 & 0.3165 & 0.3577 \\
\bottomrule
\end{tabular}%
}
\caption{Normalized PRR (higher is better) of unsupervised uncertainty estimators on Qwen2.5-7B.
Sample-based methods use 10 samples. The best result in each column is shown in bold, and the second best in blue.}
\label{tab:estimators_comparison_2}
\end{table*}
 
\subsection{Confidence intervals}
\label{app:ci}
 
We quantify the variability of PRR with a nonparametric bootstrap over evaluation instances: for each dataset, we resample instances with replacement 5000 times, recompute
PRR for every estimator on the same resample, and report the standard errors.
Differences between estimators are computed on paired resamples. Generation is greedy, so there
is no decoding randomness in single-pass estimators; sample-based baselines use the 10 samples
provided by LM-Polygraph. Table~\ref{tab:llama_with_ci} reports the intervals for the attention-based
estimators and the strongest baselines.

\begin{table*}[t]
\centering
\tiny
\setlength{\tabcolsep}{3pt}
\begin{tabular}{lc|cccc|cc}
\toprule
 & & \multicolumn{4}{c|}{\textbf{QA}} & \multicolumn{2}{c}{\textbf{Summarization}} \\
\cmidrule(lr){3-6} \cmidrule(lr){7-8}
\textbf{Estimator} & \textbf{Category} & \textbf{TriviaQA} & \textbf{TruthfulQA} & \textbf{MMLU} & \textbf{CoQA} & \textbf{SAMSum} & \textbf{XSum} \\
\midrule
MaximumSequenceProbability & Single pass & $0.25\pm{.02}$ & $0.32\pm{.06}$ & $\textbf{0.55}\pm{.02}$ & $0.29\pm{.02}$ & $-0.14\pm{.07}$ & $-0.39\pm{.11}$ \\
MeanTokenEntropy & Single pass & $\underline{0.57}\pm{.02}$ & $\textbf{\underline{0.42}}\pm{.05}$ & $0.24\pm{.03}$ & $0.26\pm{.02}$ & $0.50\pm{.05}$ & $\textbf{0.39}\pm{.08}$ \\
Perplexity & Single pass & $0.11\pm{.02}$ & $0.29\pm{.05}$ & $0.52\pm{.02}$ & $0.27\pm{.02}$ & $0.51\pm{.04}$ & $\textbf{0.41}\pm{.08}$ \\
CCP \cite{fadeeva-etal-2023-lm} & Single pass & $0.10\pm{.02}$ & $0.16\pm{.06}$ & $\textbf{\underline{0.56}}\pm{.02}$ & $0.26\pm{.02}$ & $-0.15\pm{.07}$ & $-0.30\pm{.11}$ \\
\midrule
NLI-DegMat \cite{lin2024generatingconfidenceuncertaintyquantification} & Sample-based & $\textbf{0.58}\pm{.02}$ & $0.21\pm{.06}$ & $0.44\pm{.02}$ & $\textbf{0.40}\pm{.02}$ & $0.39\pm{.05}$ & $\textbf{0.32}\pm{.09}$ \\
NLI-Eccentricity \cite{lin2024generatingconfidenceuncertaintyquantification} & Sample-based & $0.56\pm{.02}$ & $0.20\pm{.06}$ & $0.45\pm{.02}$ & $0.38\pm{.02}$ & $0.34\pm{.05}$ & $0.04\pm{.10}$ \\
NLI-EigValLaplacian \cite{lin2024generatingconfidenceuncertaintyquantification} & Sample-based & $\textbf{0.60}\pm{.02}$ & $0.20\pm{.06}$ & $0.47\pm{.02}$ & $\textbf{0.41}\pm{.02}$ & $0.39\pm{.05}$ & $\textbf{0.38}\pm{.09}$ \\
EigenScore \cite{chen2024insidellmsinternalstates} & Sample-based & $0.36\pm{.02}$ & $0.05\pm{.06}$ & $0.41\pm{.03}$ & $0.35\pm{.02}$ & $0.07\pm{.06}$ & $-0.21\pm{.10}$ \\
LUQ \cite{zhang2024luqlongtextuncertaintyquantification} & Sample-based & $\textbf{0.59}\pm{.02}$ & $0.18\pm{.06}$ & $0.44\pm{.02}$ & $0.36\pm{.02}$ & $0.29\pm{.05}$ & $\textbf{0.36}\pm{.08}$ \\
LexicalSimilarity (ROUGE-L) \cite{fomicheva-etal-2020-unsupervised} & Sample-based & $0.48\pm{.02}$ & $0.16\pm{.06}$ & $0.41\pm{.03}$ & $0.36\pm{.02}$ & $0.25\pm{.06}$ & $0.20\pm{.09}$ \\
SAR \cite{duan2024shiftingattentionrelevancepredictive} & Sample-based & $0.57\pm{.02}$ & $0.18\pm{.06}$ & $0.41\pm{.03}$ & $0.33\pm{.02}$ & $0.29\pm{.06}$ & $\textbf{0.31}\pm{.09}$ \\
SemanticDensity \cite{qiu2024semanticdensityuncertaintyquantification} & Sample-based & $0.56\pm{.02}$ & $0.08\pm{.06}$ & $0.16\pm{.03}$ & $0.11\pm{.02}$ & $0.08\pm{.06}$ & $0.11\pm{.10}$ \\
SemanticEntropy \cite{kuhn2023semanticuncertaintylinguisticinvariances} & Sample-based & $0.20\pm{.03}$ & $0.21\pm{.06}$ & $0.47\pm{.02}$ & $0.37\pm{.02}$ & $0.25\pm{.06}$ & $-0.19\pm{.11}$ \\
\midrule
AttentionScore \cite{Sriramanan2024_attention_score} & Single pass & $0.08\pm{.02}$ & $-0.05\pm{.06}$ & $0.21\pm{.03}$ & $0.09\pm{.02}$ & $-0.37\pm{.07}$ & $-1.00\pm{.11}$ \\
RAUQ \cite{vazhentsev2025uncertaintyawareattentionheadsefficient} & Single pass & $0.26\pm{.02}$ & $0.34\pm{.05}$ & $\textbf{\underline{0.56}}\pm{.02}$ & $0.31\pm{.02}$ & $0.52\pm{.05}$ & $\textbf{0.30}\pm{.08}$ \\
LapEigvalRAUQ \cite{binkowski-etal-2025-hallucination} & Single pass & $0.12\pm{.02}$ & $0.33\pm{.06}$ & $\textbf{\underline{0.56}}\pm{.02}$ & $\underline{0.31}\pm{.02}$ & $0.53\pm{.04}$ & $0.29\pm{.08}$ \\
LogDetRAUQ \cite{Sriramanan2024_attention_score} & Single pass & $0.11\pm{.02}$ & $0.33\pm{.05}$ & $\textbf{\underline{0.56}}\pm{.02}$ & $\underline{0.31}\pm{.02}$ & $0.53\pm{.04}$ & $0.30\pm{.08}$ \\
\midrule
\textbf{SAUQ (gen.\ only)} & Single pass & $0.50\pm{.02}$ & $0.30\pm{.05}$ & $\textbf{0.54}\pm{.02}$ & $0.31\pm{.02}$ & $0.41\pm{.05}$ & $0.25\pm{.09}$ \\
\textbf{SAUQ (Ours)} & Single pass & $ 0.29\pm{.02}$ & $\textbf{0.38}\pm{.05}$ & $\textbf{\underline{0.56}}\pm{.02}$ & $\underline{0.31}\pm{.02}$ & $\textbf{\underline{0.59}}\pm{.04}$ & $\textbf{\underline{0.42}}\pm{.07}$ \\
\bottomrule
\end{tabular}
\caption{Normalized PRR (higher is better) of unsupervised uncertainty estimators on LLaMA-3.1-8B, mean $\pm$ standard error
(stratified bootstrap over evaluation instances, 5000 resamples; one greedy decode per instance, 10 samples for sample-based methods).
Bold: not significantly worse than the best result in the column.}
\label{tab:llama_with_ci}
\end{table*}

\subsection{Calibration}
\label{app:ece}
 
PRR measures the quality of the ranking induced by an uncertainty score, not its calibration; for this reason, we do not expect it to be calibrated. For completeness,
Table~\ref{tab:ece} reports the expected calibration error of single-pass estimators.
 
\begin{table}[!ht]
\centering
\small
\begin{tabular}{lcc}
\toprule
\textbf{Estimator} & \textbf{QA (mean ECE)} & \textbf{Summ.\ (mean ECE)} \\
\midrule
Perplexity         & 0.237 & 0.558 \\
AttentionScore     & 0.266 & 0.551 \\
MaxSeqProbability  & 0.265 & 0.685 \\
MeanTokenEntropy   & \textbf{0.162} & 0.505 \\
RAUQ               & 0.187 & 0.522 \\
SAUQ (Ours)        & 0.194 & \textbf{0.490} \\
\bottomrule
\end{tabular}
\caption{Expected calibration error (lower is better) of single-pass estimators on LLaMA-3.1-8B,
averaged over the QA and summarization datasets.}
\label{tab:ece}
\end{table}

\subsection{Head selection and routing signal}
\label{app:selection-grid}
 
SAUQ enters the estimator at two places: the rule that selects one head per layer, and the
signal computed on that head inside the recurrence. To separate the two, we vary them
independently on TruthfulQA and SAMSum (Table~\ref{tab:selection-grid}). Rows are selection rules,
columns are signals; (Gen) denotes the generated-token block, (Full) the full prefix, and
(Eucl) and (Intrinsic) the counting and intrinsic geometries. Our configuration is
generated-Euclidean selection with the full-intrinsic signal.
 
\begin{table}[h]
\centering
\small
\setlength{\tabcolsep}{4pt}
\begin{tabular}{l|cccc}
\toprule
\multicolumn{5}{c}{\textbf{TruthfulQA}} \\
\midrule
Selection $\downarrow$ / Signal $\rightarrow$ & gen-Eucl & full-Eucl & full-intr. & full-intr.-cent. \\
\midrule
gen-Euclidean      & 0.3056 & 0.3494 & \underline{0.3767} & \textbf{0.3876} \\
full-Euclidean     & 0.3203 & 0.3507 & 0.3761 & 0.3759 \\
full-intrinsic     & 0.3078 & 0.3508 & 0.3482 & 0.3429 \\
full-intr.-cent.   & 0.3093 & 0.3508 & 0.3482 & 0.3429 \\
\midrule
\multicolumn{5}{c}{\textbf{SAMSum}} \\
\midrule
gen-Euclidean      & 0.4948 & 0.5384 & \underline{0.5924} & \textbf{0.6028} \\
full-Euclidean     & 0.4398 & 0.5375 & 0.5358 & 0.5371 \\
full-intrinsic     & 0.4916 & 0.5371 & 0.5144 & 0.5339 \\
full-intr.-cent.   & 0.5021 & 0.5373 & 0.5054 & 0.5335 \\
\bottomrule
\end{tabular}
\caption{PRR for every combination of head-selection rule (rows) and routing signal (columns).
Underlined: the configuration reported in the submitted version; bold: best per dataset.}
\label{tab:selection-grid}
\end{table}
 
Three observations follow. First, the intrinsic geometry helps on the signal axis: with the
selection rule fixed, replacing the full-Euclidean signal with the full-intrinsic one improves PRR. Second, the centered PR is marginally better than
the full PR on both datasets, consistent with Section~\ref{sec:empirical}. Third, the intrinsic geometry does not help on the
selection axis.

\paragraph{Cost of intrinsic selection.}
Selection screens all $H$ heads in every monitored layer, whereas the signal is computed only for
the selected head. Intrinsic selection would require the TDO of the full $(m+T)\times(m+T)$ matrix
for every head, i.e.\ $H$ times the cost of the signal, while the Euclidean proxy only uses the
$T\times T$ generated block.

\subsection{Length-controlled analysis}
\label{app:length}
 
Summarization and QA differ in task, input length, and output length, so the summarization gains
of Section~\ref{sec:uq} could reflect any of these. We stratify the difference
$\Delta = \mathrm{PRR}_{\mathrm{SAUQ}} - \mathrm{PRR}_{\mathrm{RAUQ}}$ by input and output length on
LLaMA-3.1-8B, using a stratified bootstrap with 5{,}000 resamples. Table~\ref{tab:length} reports the
pooled difference and the difference after controlling for output length, and for input and output
length jointly.
 
\begin{table}[h]
\centering
\small
\setlength{\tabcolsep}{4pt}
\begin{tabular}{llcccc}
\toprule
\textbf{Dataset} & \textbf{Task} & $\boldsymbol\Delta$ \textbf{pooled} & \textbf{95\% CI} & \textbf{ctrl.\ output} & \textbf{ctrl.\ input$\times$output} \\
\midrule
XSum       & Summ. & $+0.121$ & $[+0.037, +0.212]$ & $+0.078$ & $+0.060$ \\
SAMSum     & Summ. & $+0.066$ & $[+0.010, +0.121]$ & $+0.013$ & $+0.020$ \\
TruthfulQA & QA    & $+0.040$ & $[-0.004, +0.084]$ & $+0.023$ & $+0.014$ \\
TriviaQA   & QA    & $+0.028$ & $[-0.001, +0.058]$ & $+0.047$ & $+0.056$ \\
CoQA       & QA    & $-0.006$ & $[-0.015, +0.004]$ & $-0.013$ & $-0.012$ \\
MMLU       & QA    & $-0.001$ & $[-0.005, +0.004]$ & $-0.016$ & $-0.022$ \\
\bottomrule
\end{tabular}
\caption{Difference in PRR between SAUQ and RAUQ on LLaMA-3.1-8B, pooled and after stratifying by
output length and by input $\times$ output length.}
\label{tab:length}
\end{table}
 
Roughly half of the summarization gain is explained by generation length: the summarization mean
of $\Delta$ drops from $0.093$ to $0.046$ under the output-length control, while the input-length
control leaves it unchanged ($0.093$). The residual is carried by XSum, where it survives both
controls and grows with context: across input-length quartiles it is $+0.025$, $+0.053$, $+0.152$
and $+0.170$. On SAMSum, whose inputs span a narrower range of lengths, the advantage does not
survive the output-length control ($+0.066 \to +0.013$). We therefore attribute to long-range
routing only the residual XSum effect.

\subsection{Choice of measure}
\label{app:measures}
\begin{figure*}[t]
\centering
\footnotesize

\setlength{\tabcolsep}{3pt}

\begin{tabular}{ccc}

\begin{minipage}[t]{0.31\textwidth}
    \centering
    \textbf{A. Uniform}\\
    \includegraphics[width=\linewidth]{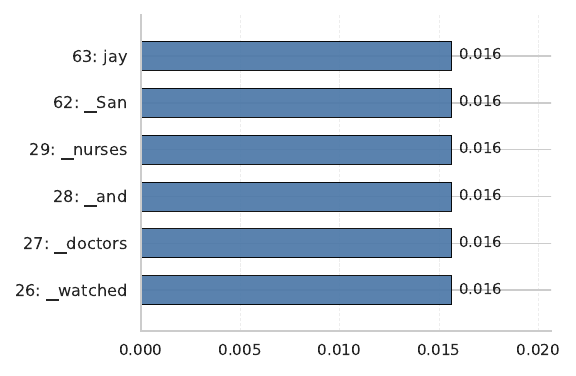}
\end{minipage}
&
\begin{minipage}[t]{0.31\textwidth}
    \centering
    \textbf{B. Positional}\\
    \includegraphics[width=\linewidth]{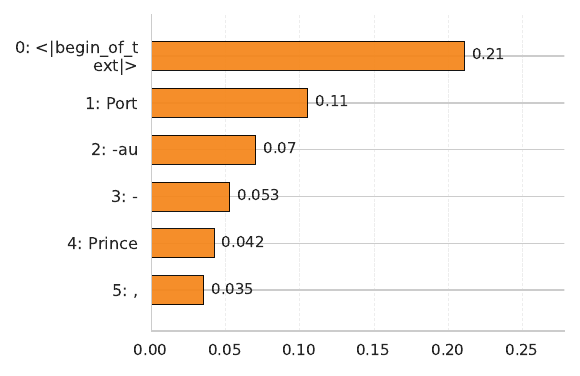}
\end{minipage}
&
\begin{minipage}[t]{0.31\textwidth}
    \centering
    \textbf{C. Zipf-Law}\\
    \includegraphics[width=\linewidth]{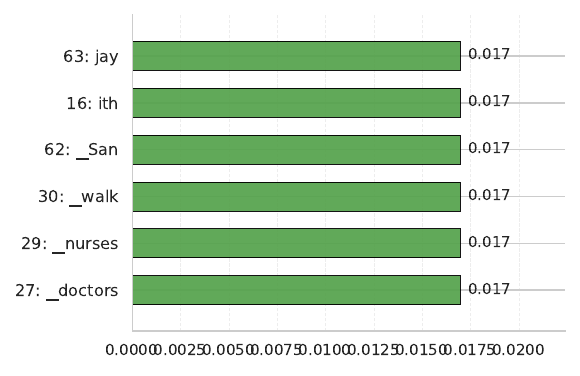}
\end{minipage}
\\[0.8em]

\begin{minipage}[t]{0.31\textwidth}
    \centering
    \textbf{D. Unembedding}\\
    \includegraphics[width=\linewidth]{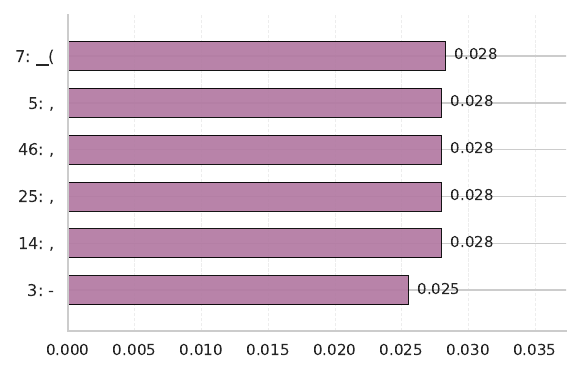}
\end{minipage}
&
\begin{minipage}[t]{0.31\textwidth}
    \centering
    \textbf{E. Intrinsic}\\
    \includegraphics[width=\linewidth]{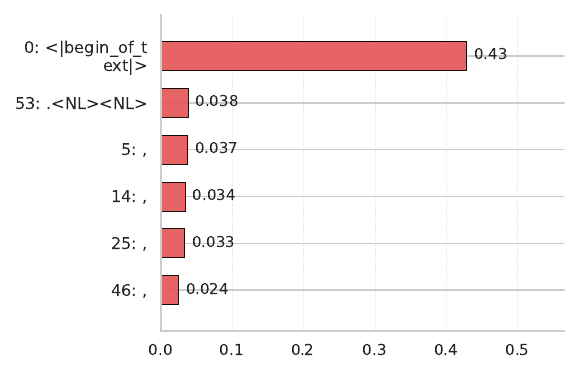}
\end{minipage}
&
\begin{minipage}[t]{0.31\textwidth}
    \centering
    \textbf{}\\[0.45em]

    \begin{tabular}{lcc}
    \toprule
    \textbf{Measure} & \textbf{QA} & \textbf{Summ.} \\
    \midrule
    Uniform    & 0.3660 & 0.4204 \\
    Positional  & \textbf{0.4027} & \textcolor{blue}{0.4771} \\
    Zipf-Law    & 0.3718 & 0.4223 \\
    Unembedding     & 0.3755 & 0.4202 \\
    Intrinsic   & \textcolor{blue}{0.3831} & \textbf{0.5069} \\
    \bottomrule
    \end{tabular}
\end{minipage}

\end{tabular}

\caption{
Top tokens receiving the largest probability mass under each measure, and mean PRR for QA and summarization of SAUQ with these masses. Input (first 100 chars):  \textit{Port-au-Prince, Haiti (CNN) -- Earthquake victims, writhing in pain and grasping at life, watched doctors and nurses...}.
}
\label{fig:sauq-masses}

\end{figure*}

All measures below are normalized to sum to one over the $n$ tokens of the prefix. In all cases
$\pi=\frac1n\mathbf 1$.
In Figure \ref{fig:sauq-masses},  we show an example on how the mass are distributed.
 
\begin{itemize}
    \item \textbf{Uniform:} $\mu_j = 1/n$.
    \item \textbf{Positional:} $\nu_l = l^{-\gamma}$ with $\gamma=1$. Early positions, including the
    sink, receive large mass; since keys enter $G_\nu$ through $\Omega_\nu^{-1}$, their columns are
    down-weighted.
    \item \textbf{Zipf-Law (in-context frequency):} $\nu_l = \big|\{k\le N : x_k = x_l\}\big|$, the number of
    occurrences of token $x_l$ in the sequence. Tokens that occur once receive the same mass, so on
    short inputs this measure is close to uniform.
    \item \textbf{Unembedding:} $\nu_l = w_{x_l}$ with $w_t = \exp\!\big(\sigma p_t - \max_{t'}\sigma p_{t'}\big)$, where
    $p_t = (W_U[t]-\bar w)^\top v_1$, $\bar w$ is the mean row of $W_U$, $v_1$ the top right singular
    vector of the centered $W_U$, and $\sigma\in\{\pm1\}$ is chosen so that $\sigma p_t$ is
    negatively correlated with the token id.
    \item \textbf{Intrinsic:} $\mu = A^\top\pi$, the average attention row of the head.
\end{itemize}

\subsection{Resource details.}
All experiments were run on a single NVIDIA L40S GPU with 48\,GB of memory, sufficient to hold
8-billion-parameter models. Running the full SAUQ evaluation on the
unsupervised benchmark takes approximately 8 hours per model. The routing signal
requires only $\operatorname{tr}(S)$ and $\operatorname{tr}(S^2)=\|S\|_F^2$ of the symmetric matrix
$S_{\mu,\pi}$, so no eigendecomposition is needed; forming $S_{\mu,\pi}$ costs
$\mathcal O(n_t^3)$ per selected head and step with dense products.

\newpage
\end{document}